\documentclass{article}
\PassOptionsToPackage{table,dvipsnames}{xcolor}
\usepackage{iclr2027_conference,times}

\usepackage[utf8]{inputenc} 
\usepackage[T1]{fontenc}    
\usepackage{hyperref}       
\usepackage{url}            
\usepackage{booktabs}       
\usepackage{amsfonts}       
\usepackage{wrapfig,lipsum}
\usepackage{amsthm,mathrsfs,amsmath,amssymb}
\usepackage{caption}
\usepackage{subcaption} 
\usepackage{graphicx,graphics}
\usepackage{svg}
\usepackage[noend]{algpseudocode}  
\usepackage{algorithm}
\usepackage{algorithmicx}
\usepackage{bbm}
\usepackage{bm}
\usepackage{tikz}
\usetikzlibrary{3d, calc}
\usepackage{svg}
\usepackage{xspace}
\graphicspath{{figures/}}
\usepackage[capitalize,noabbrev]{cleveref}
\usepackage{xcolor}
\definecolor{citeblue}{RGB}{0,70,140}
\hypersetup{
    colorlinks=true,
    citecolor=citeblue,
    linkcolor=citeblue,
    urlcolor=citeblue,
}
\usepackage{paralist} 
\usepackage{multirow, makecell, colortbl, hhline}
\usepackage{dirtytalk}
\usepackage{enumitem}

\DeclareRobustCommand{\cactustitlemark}{%
  \tikz[x=0.95em,y=0.95em,baseline=-0.10em]{%
    \fill[ForestGreen!85!black]
      (0.28,0) -- (0.48,0) -- (0.48,0.37)
      -- (0.60,0.37) .. controls (0.73,0.37) and (0.77,0.44) .. (0.77,0.55)
      -- (0.77,0.72) .. controls (0.77,0.82) and (0.62,0.82) .. (0.62,0.72)
      -- (0.62,0.55) -- (0.48,0.55) -- (0.48,0.88)
      .. controls (0.48,1.04) and (0.28,1.04) .. (0.28,0.88)
      -- (0.28,0.42) -- (0.15,0.42) -- (0.15,0.62)
      .. controls (0.15,0.72) and (0,0.72) .. (0,0.62)
      -- (0,0.42) .. controls (0,0.29) and (0.06,0.24) .. (0.18,0.24)
      -- (0.28,0.24) -- cycle;
    \draw[white!35!ForestGreen,line width=0.028em,line cap=round]
      (0.37,0.10) -- (0.37,0.87)
      (0.075,0.43) -- (0.075,0.61)
      (0.695,0.55) -- (0.695,0.71);
    \draw[ForestGreen!85!black,line width=0.023em,line cap=round]
      (0.28,0.73) -- (0.24,0.76)
      (0.48,0.81) -- (0.52,0.84)
      (0.77,0.63) -- (0.81,0.65);
    \fill[ForestGreen!85!black] (0.38,-0.17) circle[radius=0.10em];
  }%
}

\newcommand{\ourmethod}{\texttt{CACTUS}\xspace}
\newcommand{\slr}{$(\bfS + \bfL)$\xspace}

\definecolor{aurometalsaurus}{rgb}{0.43, 0.5, 0.5}
\definecolor{britishracinggreen}{rgb}{0.0, 0.26, 0.15}
\definecolor{burntumber}{rgb}{0.54, 0.2, 0.14}
\definecolor{cobalt}{rgb}{0.0, 0.28, 0.67}
\definecolor{bulgarianrose}{rgb}{0.28, 0.02, 0.03}
\definecolor{ceruleanblue}{rgb}{0.16, 0.32, 0.75}
\definecolor{darkgreen}{RGB}{0,128,0}
\definecolor{darkred}{rgb}{0.74, 0.2, 0.24}

\newcommand{\bfU}{\mathbf{U}}
\newcommand{\bfV}{\mathbf{V}}

\newcommand{\bfY}{\mathbf{Y}}
\newcommand{\bfM}{\mathbf{M}}

\newcommand{\bfH}{\mathbf{H}}
\newcommand{\bfG}{\mathbf{G}}
\newcommand{\bfI}{\mathbf{I}}
\newcommand{\bfE}{\mathbf{E}}
\newcommand{\bfP}{\mathbf{P}}
\newcommand{\bfB}{\mathbf{B}}
\newcommand{\bfQ}{\mathbf{Q}}

\newcommand{\bfA}{\mathbf{A}}

\newcommand{\bfDelta}{\mathbf{\Delta}}
\newcommand{\bfX}{\mathbf{X}}

\newcommand{\bfW}{\mathbf{W}}
\newcommand{\bfWhat}{\widehat{\mathbf{W}}}

\newcommand{\bfS}{\mathbf{S}}
\newcommand{\bfL}{\mathbf{L}}
\newcommand{\bfR}{\mathbf{R}}

\newcommand{\calC}{\mathcal{C}}
\newcommand{\calL}{\mathcal{L}}
\newcommand{\rk}[1]{\operatorname{rank}\left({#1}\right)}

\newcommand{\diag}[1]{\operatorname{diag}\left({#1}\right)}

\newcommand{\Tr}[1]{\operatorname{Tr}\left({#1}\right)}

\newcommand{\R}{\mathbb{R}}

\newtheorem{thm}{Theorem}
\crefname{thm}{Theorem}{Theorems}
\Crefname{thm}{Theorem}{Theorems}

\usepackage[dvipsnames]{xcolor}
\usepackage{tikz}
\usetikzlibrary{shapes.geometric, arrows}
\usetikzlibrary{positioning}
\usetikzlibrary {arrows.meta}
\usetikzlibrary{decorations.pathreplacing,calligraphy}

\iclrfinalcopy
\usepackage{enumitem}

\title{%
    Two for the Price of One: Curvature-Aware Pruning Preserves Influence Functions Too\nobreak\hspace{0.24em}\cactustitlemark
}

\author{%
  \makebox[\dimexpr\textwidth-2\tabcolsep\relax][c]{%
    \begin{tabular}{c}
      Mehdi Makni$^{1}$\thanks{Corresponding author: \texttt{mmakni@mit.edu}}
      \quad Natalia Ponomareva$^{2}$
      \quad Rahul Mazumder$^{1}$ \\
      {\normalfont $^{1}$Operations Research Center, Massachusetts Institute of Technology} \\
      {\normalfont $^{2}$Google Research}
    \end{tabular}%
  }%
}
\begin{document}
\maketitle
\begin{abstract}
Post-training pruning enables efficient deployment, but conventional methods do not explicitly preserve internal-gradient alignment between sparse and dense models.
We study novel connections between downstream performance of compressed models and the ability to retain gradient information via influence-function preservation. 
Central to this connection is a curvature-aware objective that accounts for both layer activations and backpropagated gradients. Despite its promise, this objective remains underexplored in pruning due to a lack of efficient solvers.
We introduce \ourmethod, a novel ADMM-based optimizer supporting both pruning and sparse-plus-low-rank decompositions of pre-trained weights. Compared to LLM-Surgeon which targets the same objective, \ourmethod can achieve over an order-of-magnitude speedup.
At $50\%$ pruning of Qwen3-4B, \ourmethod reduces the C4 perplexity gap to the dense model by $45\%$, improves compressed-to-dense per-sample gradient cosine similarity on WikiText-2 from $0.670$ to $0.755$, and runs over $28\times$ faster than LLM-Surgeon on substantially weaker hardware.
By preserving influence functions, \ourmethod supports effective data selection. On Llama-3.2-3B, fine-tuning a fresh dense model on the $5\%$ of data selected by \ourmethod achieves $49.44\%$ accuracy on BIG-Bench Hard, compared with $47.22\%$ using ALPS and $49.54\%$ using dense-model selection. Finally, integrating \ourmethod with LoGRA, a recent gradient-projection strategy, on Qwen3-4B with MLP layers pruned to 2:4 sparsity and a rank-64 correction retains a $0.711$ cosine alignment with dense LoGRA-projected C4 gradients.
\end{abstract}

\section{Introduction}
\label{sec:introduction}


Large Language Models (LLMs)~\citep{achiam2023gpt,dubey2024llama,team2023gemini} have driven advances in reasoning, code, and mathematics, but their sheer size, often measured in billions of parameters, makes both deployment and downstream inference expensive. These costs have motivated \textbf{two parallel research threads} that have advanced largely independently.

\underline{\textbf{The first thread}} is \emph{post-training pruning}. Methods such as SparseGPT~\citep{frantar2023sparsegpt}, WANDA~\citep{sun2023simple}, and ALPS~\citep{meng2024alps} use a small calibration dataset to minimize the \emph{layerwise-reconstruction error} ${\|\mathbf{X}\mathbf{W} - \mathbf{X}\widehat{\mathbf{W}}\|_F^2}$ subject to a structural constraint, where $\mathbf{X}$ contains input activations and $\mathbf{W}$ and $\widehat{\mathbf{W}}$ are the dense and compressed weights, and $\|\cdot\|_F$ is the Frobenius norm. \emph{Pruning} (pure sparsity) zeroes out the weights in an unstructured pattern or in hardware-friendly $n$:$m$ patterns that keep $n$ nonzeros in each contiguous block of $m$ weights; the common $2$:$4$ case is natively accelerated by recent NVIDIA sparse Tensor Cores~\citep{zhou2021learning}. Sparse-plus-low-rank (SLR) compression, denoted in this paper by \slr, combines a sparse backbone $\bfS$ with a low-rank correction $\bfL$ under the same framework, improving compression-utility tradeoffs over either component alone~\citep{zhang2024oats,makni2025hasslefree,makni20263basil}.

The \emph{layerwise-reconstruction} procedure above matches each layer's outputs given its input activations and is therefore \emph{gradient-blind} to the LLM loss (see \cref{fig:curvature-matters}). LLM-Surgeon~\citep{ouderaa2024llmsurgeon} (pruning) and YAQA~\citep{tseng2026modelpreserving} (quantization) additionally incorporate downstream gradients from the calibration dataset. This curvature-aware objective improves model quality, but its optimization can require multiple GPU hours per model (LLM-Surgeon in \cref{tab:main-sweep}), limiting its adoption. Our solver minimizes the curvature-aware loss function more than an \emph{order of magnitude} faster than LLM-Surgeon while improving perplexity (\cref{tab:main-sweep}). It also extends beyond pure pruning to curvature-aware \slr compression. \Cref{fig:l8b-24-r256} studies Llama-3-8B with $2$:$4$ sparsity and a rank-256 correction, matching the \slr configuration with serving speed and memory gains reported by~\citet{mozaffari2024slope}.

\underline{\textbf{The second thread}} is \emph{efficient influence function (IF) estimation}. Influence functions, with early roots in robust statistics~\citep{hampel1974influence}, estimate the counterfactual effect of a training example via $\mathcal{I}(z, z') = \nabla\ell(z')^\top \mathcal{H}^{-1}\nabla\ell(z)$, where $z$ and $z'$ are training and query examples, $\nabla\ell$ is the per-example parameter gradient, and $\mathcal{H}$ is the training-loss Hessian. Applications include data attribution~\citep{tracin,park2023trak}, training-data selection~\citep{xia2024less}, model debugging and provenance~\citep{grosse2023studying}, and machine unlearning. At LLM scale, evaluating this expression is prohibitively expensive. Existing methods reduce the cost on two fronts: (i) inverse-Hessian approximations (Kronecker-Factored Approximate Curvature (KFAC)~\citep{martens2015kfac,grosse2023studying}, its eigenvalue-corrected variant (EK-FAC), identity-Hessian as in TracIn~\citep{tracin}, Adam-style preconditioners), and (ii) on the gradient side: random or architecture-aware gradient sketches (TRAK~\citep{park2023trak}, Low-rank Gradient Approximation (LoGRA)~\citep{choe2025whatlogra}); see \cref{sec:related-work}. These methods still require \emph{per-sample gradient computations}, which remain costly for large models and datasets.

IProX~\citep{chen2026influencepreservingiprox} introduced \emph{compressed-model proxies}: with the idea of computing influence estimates using a smaller surrogate model whose gradients mimic those of the dense model. They show that minimizing the curvature-aware objective, also used by LLM-Surgeon, helps preserve per-sample gradients. IProX instantiates this idea using low-rank proxies with a two-stage pipeline: (i) each weight matrix is approximated by a product of two smaller matrices, obtained through a closed-form weighted singular value decomposition (SVD), termed influence-preserving SVD (IPSVD); and (ii) gradient-alignment distillation further adapts the low-rank proxy.


In this paper, we study the connection between these \underline{\textbf{two different research threads}} for pruning. The bridge is the curvature-aware objective in \cref{eq:kfac-loss}, and the enabler is a single algorithm that makes it practical for both high-quality pruning and influence-preserving sparse proxies. Here, influence preservation means retaining the dense model's gradient-based influence estimates after compression.
This is the sense in which the paper offers ``Two for the Price of One.'' 

Our \textbf{main contributions} can be summarized as follows.

\begin{enumerate}[noitemsep,topsep=2pt,parsep=3pt,partopsep=3pt,leftmargin=*]
\item \textbf{Efficient Curvature-Aware Compression.} We introduce \ourmethod\footnote{\underline{\textbf{C}}urvature-\underline{\textbf{A}}ware \underline{\textbf{C}}ompression with \underline{\textbf{T}}ractable \underline{\textbf{U}}pdates via \underline{\textbf{S}}ylvester Steps: Our method uses iterative closed-form updates to find high-quality solutions for \cref{eq:kfac-loss}. One key update (smooth subproblem in ADMM) solves a Sylvester equation $\bfH\bfS\bfG + \tfrac{\rho}{2}\bfS = \text{RHS}$ in closed-form at matmul cost, hence the name of the algorithm--see more details in \cref{sec:algorithm}.}, a highly-efficient algorithm tailored for \cref{eq:kfac-loss} with convergence guarantees. It handles both pure pruning (\ourmethod-prune) and sparse-plus-low-rank (\ourmethod-SLR) constraints, and runs over an order of magnitude faster than LLM-Surgeon on substantially weaker hardware on Qwen3-4B. The updates are derived in \cref{sec:model-compression}. With this new unifying algorithm, we make contributions in model compression on two fronts: preserving downstream utility and influence function preservation.
\item \textbf{Compression Quality.}
On Llama-3.2-1B at $50\%$ compression, \ourmethod reduces the C4 perplexity gap to the dense model by $25\%$ for pure pruning and $28\%$ for \slr decompositions relative to the strongest corresponding baseline. Additionally, it maintains its advantage at different compression ratios, architectures, and model sizes (see \cref{tab:main-sweep,tab:main-sweep-appendix-tracin-q4b,tab:main-sweep-appendix-l3b-ppl,tab:main-sweep-appendix-l8b-ppl,tab:main-sweep-appendix-l8b-24-r256}). In particular, we show in \cref{tab:qwen38-27b-2to4} that \ourmethod reduces the C4 perplexity gap to dense by over $25\%$ compared to SparseGPT under $2{:}4$ sparsity for Qwen3.8-27B.
\item \textbf{Influence Preservation.} On Llama-3.2-1B at $50\%$ pruning, \ourmethod improves the cosine similarity between compressed and dense per-example gradients on WikiText-2 from $0.445$ (LLM-Surgeon) to $0.512$ (\cref{tab:main-sweep-appendix-tracin}). We then use pruned models with $50\%$  sparsity to select $5\%$ of data from an instruction-tuning mixture for fine-tuning a fresh Llama-3.2-3B model. \ourmethod selected subset overlaps the dense model's selection by $54.13\%$, compared with $52.24\%$ for ALPS, and yields $49.44\%$ versus $47.22\%$ BBH accuracy (\cref{tab:less-bbh-selection}). On Qwen3-4B, compressing MLP layers to $2{:}4$ sparsity with a rank-$64$ correction retains a cosine similarity of $0.711$ between compressed and dense per-example gradients on C4, measured using the same LoGRA projection.
\end{enumerate}

\section{Curvature-Aware Model Compression}
\label{sec:model-compression}
\subsection{Problem Formulation} \label{sec:problem-formulation}
We compress a pretrained LLM on a layer-by-layer basis. For a linear layer ${\bfW \in \R^{d_\text{in} \times d_\text{out}}}$, we use $N_{\text{cal}}$ calibration data points to introduce (i) its input activations $\bfX \in \R^{N_{\text{cal}} \times d_\text{in}}$, and (ii) its downstream gradients $\bfA \in \R^{N_{\text{cal}} \times d_\text{out}}$. We seek a compressed weight ${\bfWhat}$ satisfying a structural constraint ${\bfWhat \in \calC}$ (see below) that minimizes the
Kronecker-Factored Approximate Curvature (KFAC) loss:
\begin{equation}\label{eq:kfac-loss}
  \calL_\text{kfac}(\bfWhat) \;=\; \big\|\, \bfX \,(\bfWhat - \bfW)\, \bfA^\top \,\big\|_F^2
  \;=\; \Tr{\bfH \,(\bfWhat - \bfW)\, \bfG \,(\bfWhat - \bfW)^\top},
\end{equation}
where $\operatorname{Tr}$ denotes the matrix trace, $\bfH = \bfX^\top \bfX \in \R^{d_\text{in}\times d_\text{in}}$ is the input activations covariance and $\bfG = \bfA^\top \bfA \in \R^{d_\text{out}\times d_\text{out}}$ is the downstream gradients covariance. For efficiency, we collect $\bfG$ for \emph{all} layers at the start of the algorithm using a single forward-backward pass on the $N_{\text{cal}}$ calibration data points and store them on CPU. We then proceed to compress layers sequentially by constructing $\bfH$ from outputs of previously compressed layers. Compared to the statistics collected using a layerwise reconstruction error method (e.g., SparseGPT), we only add the collection of $\bfG$ for all layers \emph{once} at the start of the algorithm. We denote Eq~\eqref{eq:kfac-loss} by \emph{KFAC} loss because it can be derived using the Kronecker factorization approximation in two distinct ways
(i) a downstream compression quality perspective in LLM-Surgeon and (ii) influence function preservation perspective in IProX, see more details in \cref{app:derivation-diagram} (\cref{fig:derivation-convergence}). For compression, we consider two  structural constraints $\calC$:
\begingroup
\setlength{\leftmargini}{1.5em}
\begin{compactitem}
  \item \textbf{Pure pruning:} either unstructured sparsity $\calC_k = \{\bfS \in \R^{d_\text{in}\times d_\text{out}} : \|\bfS\|_0 \le k\}$, or $n{:}m$-structured sparsity $\calC_{n:m}$ that keeps at most $n$ non-zeros in each contiguous block of $m$ entries. Here, $\|.\|_0$ counts the number of non-zero elements. In \cref{alg:kfac-admm}, $\calC_S$ denotes either sparsity set.
  \item \textbf{Sparse + low rank \slr:}  $\calC_{\bfS+\bfL} = \calC_S + \calC_L$, i.e.\ $\bfWhat = \bfS + \bfL$ with simultaneous sparsity and rank budgets, where $\calC_L = \{\bfL : \rk{\bfL} \le r\}$ and $r$ is the rank budget.
\end{compactitem}
\endgroup
\begin{figure}[t]
\centering
\scalebox{0.75}{%
\begin{tikzpicture}[
    >=Latex,
    every node/.style = {inner sep=0pt, outer sep=0pt},
    ray/.style        = {line width=0.95pt, line cap=round,
                         -{Latex[length=3.4pt,width=2.8pt]}},
    grayray/.style    = {line width=0.8pt, line cap=round, opacity=0.45, dashed,
                         draw=black!45, -{Latex[length=3pt,width=2.5pt]}},
    meth/.style       = {draw=black!55, rounded corners=2.5pt, font=\scriptsize,
                         inner sep=3pt, fill=black!3, line width=0.5pt},
    methOur/.style    = {draw=kfacpurple!65!black, fill=kfacpurple!14,
                         rounded corners=2.5pt, font=\scriptsize\bfseries,
                         inner sep=3pt, text=kfacpurple!25!black, line width=0.7pt},
    card/.style       = {draw=black!30, rounded corners=7pt, line width=0.6pt,
                         fill=panelbg},
    cardOur/.style    = {draw=kfacpurple!55!black, rounded corners=7pt,
                         line width=0.95pt, fill=kfacpurple!4},
    progarr/.style    = {-{Latex[length=5pt,width=4pt]}, line width=1.2pt,
                         line cap=round}
]

\definecolor{kfacpurple}{RGB}{120,60,170}
\definecolor{modeldark}{RGB}{52,52,72}
\definecolor{optgreen}{RGB}{32,140,82}
\definecolor{panelbg}{RGB}{252,252,254}
\definecolor{databluebg}{RGB}{230,242,252}
\definecolor{datablue}{RGB}{40,110,190}
\definecolor{gradorangebg}{RGB}{253,236,218}
\definecolor{gradorange}{RGB}{210,95,35}
\definecolor{ghostgray}{RGB}{205,205,210}

\definecolor{rayA}{RGB}{214,40,40}
\definecolor{rayB}{RGB}{247,127,0}
\definecolor{rayC}{RGB}{212,175,55}
\definecolor{rayD}{RGB}{6,156,90}
\definecolor{rayE}{RGB}{35,116,201}
\definecolor{rayF}{RGB}{120,60,170}
\definecolor{rayG}{RGB}{202,52,146}

\node[anchor=center, font=\normalsize, align=center] at (8.9, 6.95)
   {The same $\mathbf W$ to be compressed -- 3 levels of data-awareness.};

\newcommand{\drawWmatrix}[1]{%
  \begin{scope}[shift={(#1)}]
    \def\cs{0.22}
    \foreach \i in {0,1,2,3,4}
      \foreach \j in {0,1,2,3,4}
        \fill[modeldark!85]
          (\i*\cs, \j*\cs) rectangle (\i*\cs+\cs*0.92, \j*\cs+\cs*0.92);
    \foreach \ii/\jj in {0/0, 1/2, 2/1, 0/3, 3/3, 4/2, 3/0, 1/4, 4/4, 2/3}
        \fill[black!8]
          (\ii*\cs, \jj*\cs) rectangle (\ii*\cs+\cs*0.92, \jj*\cs+\cs*0.92);
    \draw[draw=black!45, line width=0.5pt] (0,0) rectangle (5*\cs, 5*\cs);
    \node[anchor=north, font=\small\bfseries] at (0.5*5*\cs, -0.08) {$\mathbf W$};
  \end{scope}
}

%

\begin{scope}[shift={(0,0)}]
  \draw[card] (0, 1.15) rectangle (5, 6.3);

  \node[anchor=north, font=\footnotesize\bfseries] at (2.5,6.05) {Data-free};
  \node[anchor=north, font=\tiny\itshape, text=black!55] at (2.5,5.30)
     {$\times$\,$\mathbf X$\quad$\times$\,$\mathbf A$\;\; (weights only)};

  \draw[draw=ghostgray, fill=ghostgray!30, line width=0.5pt, dashed,
        rounded corners=1.5pt]
        (0.30, 3.70) rectangle (1.00, 5.10);
  \node[font=\small, text=black!45] at (0.65, 4.40) {$\mathbf X$};
  \node[anchor=north, font=\tiny\itshape, text=black!40] at (0.65, 3.68)
     {unused};

  \drawWmatrix{1.78, 3.85}

  \draw[draw=ghostgray, fill=ghostgray!30, line width=0.5pt, dashed,
        rounded corners=1.5pt]
        (3.55, 4.10) rectangle (4.65, 4.70);
  \node[font=\small, text=black!45] at (4.16, 4.41) {$\mathbf A^\top$};

  \foreach \y in {3.90, 4.15, 4.40, 4.65, 4.90}{
     \draw[grayray] (1.00, \y) -- (1.72, \y);
  }
  \foreach \y in {4.15, 4.28, 4.40, 4.52, 4.65}{
     \draw[grayray] (3.55, \y) -- (2.92, \y);
  }

  \node[anchor=north, font=\tiny\itshape, text=black!40] at (4.10, 3.68)
     {unused};

  \draw[rounded corners=3pt, draw=black!30, fill=black!4, line width=0.5pt]
     (0.40, 2.75) rectangle (4.60, 3.40);
  \node[anchor=center, font=\footnotesize] at (2.5, 3.075)
     {$\displaystyle \|\widehat{\mathbf W}-\mathbf W\|_F^2$};

  \draw[rounded corners=2.5pt, draw=black!40, fill=white, line width=0.5pt]
     (0.25, 1.75) rectangle (1.69, 2.40);
  \node[anchor=center, font=\tiny\bfseries, text=black!60]
     at (0.97, 2.25) {\textsc{s}};
  \node[anchor=center, font=\fontsize{3.5pt}{4.5pt}\selectfont, text=black!80]
     at (0.97, 1.95) {MP (magnitude pruning)};

  \draw[rounded corners=2.5pt, draw=black!40, fill=white, line width=0.5pt]
     (1.78, 1.75) rectangle (3.22, 2.40);
  \node[anchor=center, font=\tiny\bfseries, text=black!60]
     at (2.50, 2.25) {\textsc{l}};
  \node[anchor=center, font=\fontsize{3.5pt}{4.5pt}\selectfont, text=black!80]
     at (2.50, 1.95) {SVD};

  \draw[rounded corners=2.5pt, draw=black!40, fill=white, line width=0.5pt]
     (3.31, 1.75) rectangle (4.75, 2.40);
  \node[anchor=center, font=\tiny\bfseries, text=black!60]
     at (4.03, 2.25) {\textsc{s+l}};
  \node[anchor=center, font=\fontsize{3.5pt}{4.5pt}\selectfont, text=black!80]
     at (4.03, 1.95) {Robust-PCA};

  \node[anchor=south, font=\tiny\itshape, text=black!45]
     at (2.5, 1.35) {No calibration data};
\end{scope}

\draw[progarr, draw=datablue!70!black, line width=1.8pt]
   (5.20, 4.40) -- (6.20, 4.40);
\node[anchor=south, font=\small\bfseries, text=datablue!40!black]
     at (5.70, 4.55) {$+\,\mathbf X$};

\begin{scope}[shift={(6.4,0)}]
  \draw[card] (0, 1.15) rectangle (5, 6.3);

  \node[anchor=north, align=center, font=\footnotesize\bfseries] at (2.5,6.05)
     {Layerwise reconstruction\\[1pt]\footnotesize(most common)};
  \node[anchor=north, font=\tiny\itshape, text=black!55] at (2.5,5.30)
     {\checkmark\,$\mathbf X$\quad$\times$\,$\mathbf A$\;\; (gradient-blind)};

  \draw[draw=datablue!70!black, fill=databluebg,
        line width=0.55pt, rounded corners=1.5pt]
        (0.30, 3.70) rectangle (1.00, 5.10);
  \foreach \y in {0.20, 0.40, 0.60, 0.80, 1.00, 1.20} {
    \draw[draw=datablue!60, line width=0.35pt]
       (0.40, 3.70+\y) -- (0.90, 3.70+\y);
  }
  \node[font=\small\bfseries, text=datablue!30!black] at (0.65, 4.40)
     {$\mathbf X$};
  \node[anchor=north, font=\tiny\itshape, text=datablue!40!black]
     at (0.65, 3.68) {active};

  \drawWmatrix{1.78, 3.85}

  \draw[draw=ghostgray, fill=ghostgray!30, line width=0.5pt, dashed,
        rounded corners=1.5pt]
        (3.55, 4.10) rectangle (4.65, 4.70);
  \node[font=\small, text=black!45] at (4.16, 4.41) {$\mathbf A^\top$};
  \node[anchor=north, font=\tiny\itshape, text=black!40] at (4.10, 3.68)
     {unused};

  \foreach \ys/\yt in {3.80/3.90, 4.10/4.15, 4.40/4.40,
                        4.70/4.65, 4.95/4.80}{
     \draw[ray, draw=datablue!80!black]
       (1.00, \ys) .. controls (1.25, \ys) and (1.55, \yt) .. (1.76, \yt);
  }

  \foreach \y in {4.15, 4.28, 4.40, 4.52, 4.65}{
     \draw[grayray] (3.55, \y) -- (2.92, \y);
  }

  \draw[rounded corners=3pt, draw=datablue!55!black, fill=databluebg,
        line width=0.6pt]
     (0.40, 2.75) rectangle (4.60, 3.40);
  \node[anchor=center, font=\footnotesize] at (2.5, 3.075)
     {$\displaystyle
       \|{\color{datablue!30!black}\mathbf X}(\widehat{\mathbf W}-\mathbf W)\|_F^2$};

  \draw[rounded corners=2.5pt, draw=datablue!45!black,
        fill=databluebg!45, line width=0.5pt]
     (0.25, 1.75) rectangle (1.69, 2.40);
  \node[anchor=center, font=\tiny\bfseries, text=datablue!40!black]
     at (0.97, 2.25) {\textsc{s}};
  \node[anchor=center, align=center, font=\fontsize{3.5pt}{4.5pt}\selectfont, text=black!80]
     at (0.97, 1.95) {SparseGPT\\ALPS};

  \draw[rounded corners=2.5pt, draw=datablue!45!black,
        fill=databluebg!45, line width=0.5pt]
     (1.78, 1.75) rectangle (3.22, 2.40);
  \node[anchor=center, font=\tiny\bfseries, text=datablue!40!black]
     at (2.50, 2.25) {\textsc{l}};
  \node[anchor=center, font=\fontsize{3.5pt}{4.5pt}\selectfont, text=black!80]
     at (2.50, 1.95) {SVD-LLM};

  \draw[rounded corners=2.5pt, draw=datablue!45!black,
        fill=databluebg!45, line width=0.5pt]
     (3.31, 1.75) rectangle (4.75, 2.40);
  \node[anchor=center, font=\tiny\bfseries, text=datablue!40!black]
     at (4.03, 2.25) {\textsc{s+l}};
  \node[anchor=center, align=center, font=\fontsize{3.5pt}{4.5pt}\selectfont, text=black!80]
     at (4.03, 1.95) {OATS\\3BASiL};

  \node[anchor=south, font=\tiny\itshape, text=black!45]
     at (2.5, 1.35) {Uses Activations $\mathbf X$};
\end{scope}

\draw[progarr, draw=kfacpurple!70!black, line width=1.8pt]
   (11.60, 4.40) -- (12.60, 4.40);
\node[anchor=south, font=\small\bfseries, text=kfacpurple!40!black]
     at (12.10, 4.55) {$+\,\mathbf A$};

\begin{scope}[shift={(12.8,0)}]
  \draw[cardOur] (0, 1.15) rectangle (5, 6.3);

  \node[anchor=north, font=\footnotesize\bfseries, text=kfacpurple!25!black]
     at (2.5,6.05) {Curvature-aware};
  \node[anchor=north, font=\tiny\itshape, text=kfacpurple!45!black]
     at (2.5,5.30)
     {\checkmark\,$\mathbf X$\quad\checkmark\,$\mathbf A$\;\; (IF preserving)};

  \draw[draw=datablue!70!black, fill=databluebg,
        line width=0.55pt, rounded corners=1.5pt]
        (0.30, 3.70) rectangle (1.00, 5.10);
  \foreach \y in {0.20, 0.40, 0.60, 0.80, 1.00, 1.20} {
    \draw[draw=datablue!60, line width=0.35pt]
       (0.40, 3.70+\y) -- (0.90, 3.70+\y);
  }
  \node[font=\small\bfseries, text=datablue!30!black] at (0.65, 4.40)
     {$\mathbf X$};
  \node[anchor=north, font=\tiny\itshape, text=datablue!40!black]
     at (0.65, 3.68) {active};

  \drawWmatrix{1.78, 3.85}

  \draw[draw=kfacpurple!70!black, fill=kfacpurple!12,
        line width=0.55pt, rounded corners=1.5pt]
        (3.55, 4.10) rectangle (4.65, 4.70);
  \foreach \x in {0.18, 0.34, 0.50, 0.76, 0.92} {
     \draw[draw=kfacpurple!60, line width=0.35pt]
        (3.55+\x, 4.16) -- (3.55+\x, 4.64);
  }
  \node[font=\small\bfseries, text=kfacpurple!25!black] at (4.16, 4.41) {$\mathbf A^\top$};
  \node[anchor=north, font=\tiny\itshape, text=kfacpurple!40!black]
     at (4.10, 3.68) {active};

  \foreach \ys/\yt in {3.80/3.90, 4.10/4.15, 4.40/4.40,
                        4.70/4.65, 4.95/4.80}{
     \draw[ray, draw=datablue!80!black]
       (1.00, \ys) .. controls (1.25, \ys) and (1.55, \yt) .. (1.76, \yt);
  }

  \foreach \ya/\yw in {4.15/3.96, 4.28/4.18, 4.40/4.40, 4.52/4.62, 4.65/4.84}{
     \draw[ray, draw=kfacpurple!85!black, line width=0.8pt, opacity=0.92]
       (3.55, \ya) .. controls (3.30, \ya) and (3.10, \yw) .. (2.92, \yw);
  }

  \draw[rounded corners=3pt, draw=kfacpurple!55!black, fill=kfacpurple!10,
        line width=0.75pt]
     (0.40, 2.75) rectangle (4.60, 3.40);
  \node[anchor=center, font=\footnotesize] at (2.5, 3.075)
     {$\displaystyle
       \|{\color{datablue!30!black}\mathbf X}(\widehat{\mathbf W}-\mathbf W)
          {\color{kfacpurple}\mathbf A^{\!\top}}\|_F^2$};

  \draw[rounded corners=2.5pt, draw=kfacpurple!35!black,
        fill=kfacpurple!4, line width=0.5pt]
     (0.25, 1.75) rectangle (1.69, 2.40);
  \node[anchor=center, font=\tiny\bfseries, text=kfacpurple!50!black]
     at (0.97, 2.25) {\textsc{s}};
  \node[anchor=center, align=center, font=\fontsize{3.5pt}{4.5pt}\selectfont, text=black!75]
     at (0.97, 1.95) {LLM Surgeon\\ \textcolor{kfacpurple}{\ourmethod-prune}};

  \draw[rounded corners=2.5pt, draw=kfacpurple!35!black,
        fill=kfacpurple!4, line width=0.5pt]
     (1.78, 1.75) rectangle (3.22, 2.40);
  \node[anchor=center, font=\tiny\bfseries, text=kfacpurple!50!black]
     at (2.50, 2.25) {\textsc{l}};
  \node[anchor=center, font=\fontsize{3.5pt}{4.5pt}\selectfont, text=black!75]
     at (2.50, 1.95) {IProX -- IPSVD};

  \draw[rounded corners=2.5pt, draw=kfacpurple!70!black,
        fill=kfacpurple!18, line width=0.95pt]
     (3.31, 1.75) rectangle (4.75, 2.40);
  \node[anchor=center, font=\tiny\bfseries, text=kfacpurple!70!black]
     at (4.03, 2.25) {\textsc{s\,+\,l}};
  \node[anchor=center, font=\fontsize{3.5pt}{4.5pt}\selectfont, text=kfacpurple!85!black]
     at (4.03, 1.95) {\ourmethod-SLR};

  \node[anchor=south, font=\tiny\itshape, text=kfacpurple!40!black]
     at (2.5, 1.35) {Uses Activations $\mathbf X$ \emph{\&} Gradients $\mathbf A$};
\end{scope}

\end{tikzpicture}%
}
\caption{Each panel is a layer-level reconstruction problem centered on the same weight
matrix~$\mathbf W$.  Input activations from the previously compressed layer are represented by~${\mathbf X \in \R^{N_{\text{cal}} \times d_{\text{in}}}}$; the downstream gradients incoming from the next layer are represented by~${\mathbf A \in \R^{N_{\text{cal}} \times d_{\text{out}}}}$.
\emph{(Left)} \textbf{Data-free.}  Neither $\mathbf X$ nor $\mathbf A$ is
available; both sides reduce to parallel gray passes and the objective
collapses to ${\|\widehat{\mathbf W}-\mathbf W\|_F^2}$.
\emph{(Middle)} \textbf{Layerwise reconstruction (most common).}  Calibration
activations ${\mathbf X}$ light up the left side, giving
${\|\mathbf X(\widehat{\mathbf W}-\mathbf W)\|_F^2}$; outputs are matched but
the method remains \emph{gradient-blind} (SparseGPT, ALPS, 3BASiL, SVD-LLM).
\emph{(Right)} \textbf{Curvature-aware.}  Both activations and the
downstream gradients are used, yielding
${\|\mathbf X(\widehat{\mathbf W}-\mathbf W)\mathbf A^{\!\top}\|_F^2}$, as studied in this paper.}
\label{fig:curvature-matters}
\end{figure}
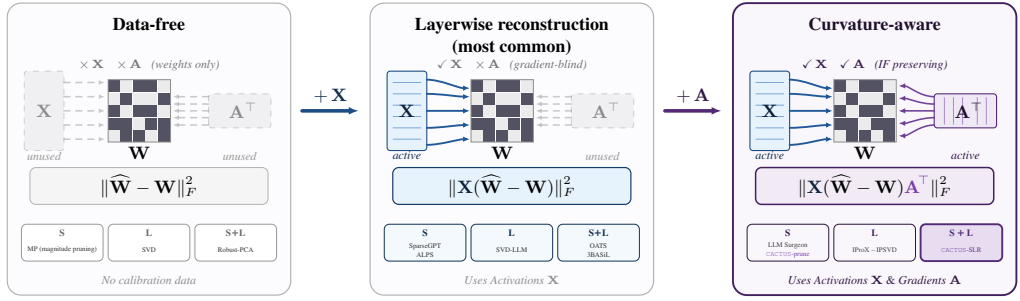

\vspace{-0.1em}
\subsection{\ourmethod Algorithm Derivation} \label{sec:algorithm}
\vspace{-0.1em}
We treat the \slr family as the general case, with pure pruning (obtained by setting rank $r = 0$) as a special case. We use the Alternating Direction Method of Multipliers (ADMM)~\citep{boyd2011distributed} for the Kronecker-Factored Approximate Curvature (KFAC) objective \eqref{eq:kfac-loss}.
Using variable splitting for the sparse component, 
our problem formulation is given as:
\begin{equation}\label{eq:splitting}
\begin{aligned}
  \min_{\bfS,\bfDelta,\bfL} \quad & \Tr{\bfH (\bfW - \bfS - \bfL) \bfG (\bfW - \bfS - \bfL)^\top} \\
  \text{s.t.}\quad & \bfS \;=\;\bfDelta, \qquad \bfDelta \in \calC_S, \quad \rk{\bfL} \le r,
\end{aligned}
\end{equation}
where $\bfDelta\in \calC_S$ is set to be sparse and $\bfS$ is its copy; $\bfL$ is the low-rank component with $\rk{\bfL} \le r$. The sparsity constraint is decoupled from the smooth objective through the consensus equality $\bfS = \bfDelta$. 
The augmented Lagrangian with scaled dual variable $\bfV$ and penalty $\rho > 0$ is
\begin{equation}\label{eq:admm-lagrangian}
  \calL_\rho(\bfS,\bfL,\bfDelta,\bfV) \;=\; \Tr{\bfH (\bfW - \bfS - \bfL) \bfG (\bfW - \bfS - \bfL)^\top} \;+\; \tfrac{\rho}{2} \big\| \bfS - \bfDelta + \bfV \big\|_F^2 \;-\; \tfrac{\rho}{2}\|\bfV\|_F^2.
\end{equation}
In practice, we damp $\bfH$ and $\bfG$ using separate ridge parameters $\lambda_H$ and $\lambda_G$ to improve conditioning (see \cref{sect:appendix-exp-details}). In what follows, $\bfH$ and $\bfG$ denote these damped factors. We perform an eigendecomposition of $\bfH$ and $\bfG$ once, then ADMM alternates across three primal blocks ($\bfS$, $\bfDelta$, $\bfL$)
one at a time followed by a dual update in $\bfV$. 
We minimize~\eqref{eq:admm-lagrangian} 
over each primal block $\bfS$, $\bfDelta\in\calC_S$ and $\rk{\bfL}\le r$ one at a time. Each primal subproblem minimization admits a closed-form solution 
and the $\bfS$- and $\bfL$-updates reuse the same eigenspaces throughout the ADMM iterations:
\begin{equation}\label{eq:h-g-eigen}
{
  \bfH \;=\; \bfQ_H \,\bm{\Lambda}_H\, \bfQ_H^\top, \qquad
  \bfG \;=\; \bfQ_G \,\bm{\Lambda}_G\, \bfQ_G^\top,
~~~~~~(\text{eigendecompositions)}}
\end{equation}
with $\bm{\Lambda}_H = \diag{\lambda_{H,1},\dots,\lambda_{H,d_\text{in}}}$ and $\bm{\Lambda}_G = \diag{\lambda_{G,1},\dots,\lambda_{G,d_\text{out}}}$.
\paragraph{(i) $\bfS$-update: Sylvester equation.}
Holding $\bfL,\bfDelta,\bfV$ fixed, the augmented Lagrangian \eqref{eq:admm-lagrangian} is a strictly convex quadratic in $\bfS$. Setting $\nabla_{\bfS}\calL_\rho = 0$ yields the matrix linear system
\begin{equation}\label{eq:s-update}
  \bfH\,\bfS\,\bfG \;+\; \tfrac{\rho}{2}\, \bfS \;=\; \bfH\,(\bfW - \bfL)\,\bfG \;+\; \tfrac{\rho}{2}\,\big(\bfDelta - \bfV\big),
\end{equation}
which is a \emph{generalized Sylvester equation} in $\bfS$ \citep{sylvester1909collected}. Crucially, this is what differentiates \ourmethod from prior ADMM-based layerwise reconstruction methods such as ALPS \citep{meng2024alps} or 3BASiL \citep{makni20263basil}: those methods only deal with a one-sided system $(\bfH+\tfrac{\rho}{2}\bfI)\bfS = \text{RHS}$ under our normalization, whereas the curvature-aware KFAC objective forces a \emph{two-sided} coupling that, naively solved, would cost $O((d_\text{in}d_\text{out})^3)$ via a vectorized Kronecker formulation $(\bfG^\top\!\!\otimes\bfH + \tfrac{\rho}{2}\bfI)\,\mathrm{vec}(\bfS) = \mathrm{vec}(\text{RHS})$, which is completely \textbf{intractable} for LLM-scale matrices.

\noindent\textbf{Closed-form solve in the two eigenbases.} The Sylvester structure admits an exact solution by substituting $\bfS = \bfQ_H \widetilde{\bfS} \bfQ_G^\top$ into \eqref{eq:s-update};
the equation \emph{decouples elementwise}:
\vspace{-0.4em}
\begin{equation}\label{eq:s-update-eigen}
  \big(\lambda_{H,i}\,\lambda_{G,j} + \tfrac{\rho}{2}\big)\,\widetilde{\bfS}_{ij}
  \;=\; \widetilde{\mathrm{RHS}}_{ij},
  \qquad
  \widetilde{\mathrm{RHS}} \;=\; \bfQ_H^\top\,\big[\bfH(\bfW-\bfL)\bfG + \tfrac{\rho}{2}(\bfDelta-\bfV)\big]\,\bfQ_G.
\end{equation}
\noindent We obtain each entry of $\widetilde{\bfS}$ by a single scalar division with $\lambda_{H,i}\lambda_{G,j}+\rho/2 > 0$---the $\bfS$-step remains nonsingular for any $\rho>0$ and positive definite $\bfH,\bfG$. The final solution is $\bfS = \bfQ_H\widetilde{\bfS}\bfQ_G^\top$.

\noindent\textbf{Algorithmic implications.} The whole inner solve thus reduces to (i) matmuls to form $\bfQ_H^\top(\cdot)\bfQ_G$, (ii) an elementwise division on a $d_\text{in}\!\times\!d_\text{out}$ matrix, and (iii) matmuls to undo the basis change. The smooth solve costs $O(d_\text{in}d_\text{out}(d_\text{in}+d_\text{out}))$ per iteration, with the eigenbases reused across all $T$ iterations and in the $\bfL$-update for whitening. The ability to solve \eqref{eq:s-update} \emph{exactly and at matmul cost} is key to making curvature-aware ADMM tractable and yields the runtimes reported in \cref{sec:experimental-results}.
\vspace{-0.4em}
\paragraph{(ii) $\bfL$-update: IPSVD \citep{chen2026influencepreservingiprox}.}
Holding $\bfS,\bfDelta,\bfV$ fixed, for the $\bfL$-update, we minimize
${\Tr{\bfH (\bfR_\bfS - \bfL) \bfG (\bfR_\bfS - \bfL)^\top}}$ (here, ${\bfR_\bfS \;=\; \bfW - \bfS}$) subject to $\rk{\bfL}\le r$. Since $\bfH$ and $\bfG$ are positive definite, the change of variables $\bfM=\bfH^{1/2}\bfL\bfG^{1/2}$ preserves rank and reduces the subproblem to ordinary low-rank approximation. Its exact minimizer is therefore
\begin{equation}\label{eq:l-update-eigen}
  \bfL^\star \;=\; \bfH^{-1/2}\, \mathrm{SVD}_r\!\big(\bfH^{1/2}\, \bfR_\bfS\, \bfG^{1/2}\big)\, \bfG^{-1/2},
\end{equation}
where $\mathrm{SVD}_r$ denotes a best rank-$r$ approximation obtained by truncated SVD, and $\bfH^{\pm 1/2} = \bfQ_H \bm{\Lambda}_H^{\pm 1/2} \bfQ_H^\top$ and $\bfG^{\pm 1/2} = \bfQ_G \bm{\Lambda}_G^{\pm 1/2} \bfQ_G^\top$ reuse eigendecompositions from \eqref{eq:h-g-eigen}.
We accept this candidate only when its movement is bounded by the smooth-block movement:
\begin{equation}\label{eq:cactus-lowrank-gate}
\bfL^{(t+1)}=
\begin{cases}
\bfL^\star, & \|\bfL^\star-\bfL^{(t)}\|_F\le C_L\|\bfS^{(t+1)}-\bfS^{(t)}\|_F,\\
\bfL^{(t)}, & \text{otherwise},
\end{cases}
\end{equation}
where $C_L\ge0$ is fixed and finite. For pure pruning ($r=0$), we keep $\bfL=\mathbf 0$.
\vspace{-1em}
\paragraph{(iii) $\bfDelta$-update: exact sparse projection.}
The only term involving $\bfDelta$ is the consensus quadratic, so we retain the top-$k$ entries by magnitude (unstructured) or the top-$n$ entries within each $m$-block ($n{:}m$-structured) and set the remaining elements to $0$. We denote this projection by $\Pi_{\calC_S}\!\big(\bfS + \bfV\big)$.
\vspace{-1em}
\paragraph{(iv) Scaled dual and penalty updates.}
Given the new penalty, we update $\bfV^{(t+1)}=\frac{\rho^{(t)}}{\rho^{(t+1)}}\big(\bfV^{(t)}+\bfS^{(t+1)}-\bfDelta^{(t+1)}\big)$. The ratio rescales the dual variable to preserve the updated Lagrange multiplier.

\begin{algorithm}[H]
\small
\caption{\ourmethod for one linear layer.}
\label{alg:kfac-admm}
\begin{algorithmic}[1]
\Require Pretrained $\bfW$, damped positive-definite factors $\bfH,\bfG$ (eigendecomposed), sparsity set $\calC_S$, rank $r$, movement bound $C_L$, positive penalties $\{\rho^{(t)}\}$, iterations $T$.
\State Initialize $\bfS^{(0)}\gets\bfW$, $\bfL^{(0)}\gets\mathbf0$, $\bfDelta^{(0)}\gets\mathbf0$, $\bfV^{(0)}\gets\mathbf0$.
\For{$t=0,1,\ldots,T-1$}
  \State $\bfS^{(t+1)}\gets$ Sylvester solve \eqref{eq:s-update} using $(\bfL^{(t)},\bfDelta^{(t)},\bfV^{(t)},\rho^{(t)})$.
    \State $\bfL^\star\gets$ weighted rank-$r$ SVD of $\bfW-\bfS^{(t+1)}$ via \eqref{eq:l-update-eigen}.
    \State $\bfL^{(t+1)}\gets\bfL^\star$ \textbf{if} $\|\bfL^\star-\bfL^{(t)}\|_F\le C_L\|\bfS^{(t+1)}-\bfS^{(t)}\|_F$ \textbf{else} $\bfL^{(t)}$.
  \State $\bfDelta^{(t+1)}\gets\Pi_{\calC_S}(\bfS^{(t+1)}+\bfV^{(t)})$.
  \State $\bfV^{(t+1)}\gets\frac{\rho^{(t)}}{\rho^{(t+1)}}\big(\bfV^{(t)}+\bfS^{(t+1)}-\bfDelta^{(t+1)}\big)$.
\EndFor
\State \Return $\bfWhat\gets\bfDelta^{(T)}+\bfL^{(T)}$.
\end{algorithmic}
\end{algorithm}
\vspace{-1.5em}
\subsection{Convergence}\label{sec:cactus-convergence}
\vspace{-0.5em}
The convergence argument follows 3BASiL~\citep{makni20263basil}. Under the conditions below, the iterates converge to a feasible sparse-plus-low-rank decomposition.
\begin{thm}[Convergence to a feasible decomposition]\label{thm:cactus-convergence}
Consider \cref{alg:kfac-admm} with positive definite $\bfH,\bfG$, exact smooth solves and sparse projections, and a fixed finite movement bound $C_L$. Suppose the positive penalties ($\rho_{t}$) are non-decreasing and satisfy $\sum_{t=0}^{\infty}1/\rho^{(t)}<\infty$. Then
\[
\bfS^{(t)},\bfDelta^{(t)}\longrightarrow\bar{\bfS}\in\calC_S,
\qquad
\bfL^{(t)}\longrightarrow\bar{\bfL},\quad \rk{\bar{\bfL}}\le r.
\]
The successive changes in $\bfS^{(t)}$ and $\bfL^{(t)}$, and the consensus residual $\bfS^{(t+1)}-\bfDelta^{(t+1)}$, are all $O(1/\rho^{(t-1)})$ in Frobenius norm.
\end{thm}
The two-sided curvature changes the smooth solve, but preserves the projection and boundedness arguments used in 3BASiL; \cref{app:cactus-convergence} gives only the changes to that proof. Increasing $\rho$ by a fixed factor greater than one at fixed intervals satisfies the summability condition (see \cref{sect:appendix-exp-details}).


\vspace{-0.4em}
\section{Experimental Results}
\vspace{-0.4em}
\label{sec:experimental-results}
\subsection{Experimental Setup}\label{sect:exp-setup}
\noindent\textbf{Models and Evaluation Protocol.}~~
We compute gradients using the language-model cross-entropy loss.
We evaluate \ourmethod on different settings and configurations using one of the following model architectures: Llama-3.2-1B/3B, Llama-2-7B, Llama-3-8B \citep{dubey2024llama}, Qwen3-4B ~\citep{yang2025qwen3} or Qwen3.8-27B. For model compression utility, we report C4~\citep{JMLR:v21:20-074} and WikiText-2 (WT-2) \citep{merity2017pointer} perplexity (PPL) at sequence length $2048$. Further, we evaluate influence function (IF) preservation under two gradient representations, since influence methods differ mainly in the inner product used to compare per-sample gradients. For every sample $x$ and linear layer $\ell$, let ${\nabla_{x,\ell}\in\R^{d_{\ell,\text{in}}\times d_{\ell,\text{out}}}}$ denote the per-sample weight gradient. The \textbf{TracIn} representation~\citep{tracin} uses the raw single-checkpoint gradient vector ${r_x^{\mathrm{T}}=\mathrm{concat}_{\ell}\,\mathrm{vec}(\nabla_{x,\ell})}$. The \textbf{KFAC-weighted} representation~\citep{mlodozeniec2024influence,grosse2023studying} uses ${r_x^{\mathrm{K}}=\mathrm{concat}_{\ell}\,\mathrm{vec}(\nabla^{\mathrm{K}}_{x,\ell})}$ with ${\nabla^{\mathrm{K}}_{x,\ell}=\bfH_\ell^{-1/2}\nabla_{x,\ell}\bfG_\ell^{-1/2}}$, where $\bfH_\ell$ and $\bfG_\ell$ are the dense model's activation and backpropagated-gradient covariance factors. We use the dense factors for both dense and compressed gradients: they are calibration-level curvature statistics, distinct from per-sample gradients, and can be precomputed once then applied on the fly. Before projection, the KFAC-weighted layerwise inner product is ${\langle\nabla^{\mathrm{K}}_{x,\ell},\nabla^{\mathrm{K}}_{q,\ell}\rangle_F = \mathrm{vec}(\nabla_{x,\ell})^\top(\bfG_\ell^{-1}\!\otimes\!\bfH_\ell^{-1})\mathrm{vec}(\nabla_{q,\ell})}$, corresponding to the standard KFAC approximation of the Hessian.

For either representation ${r_x\in\{r_x^{\mathrm{T}},r_x^{\mathrm{K}}\}}$, we store ${\tilde\nabla_x=\Pi r_x}$ using the same TRAK~\citep{park2023trak} Rademacher Johnson-Lindenstrauss (JL) sketch $\Pi$ down to $d=8192$. This projection preserves the relevant geometry: ${\langle \Pi r_x,\Pi r_q\rangle \approx \langle r_x,r_q\rangle}$ and $\|\Pi r_x\|\approx\|r_x\|$, so projected inner products and cosine similarities approximate their full-gradient counterparts. In what follows, ${\tilde\nabla^{\text{d}}_x, \tilde\nabla^{\text{c}}_x \in \R^d}$ denote the dense and compressed projected gradients of the same sample $x$ for a fixed representation, and ${\bfI^{\text{d}}, \bfI^{\text{c}} \in \R^{n_{\text{tr}}\times n_{\text{te}}}}$ denote each model's \emph{influence matrix} with entries ${\bfI_{ij} = \langle\tilde\nabla_{x_i},\tilde\nabla_{q_j}\rangle}$ (training row $i$, test column $j$). Following influence-function terminology, ``training'' rows are the source examples being ranked, while ``test'' columns are the query examples (distinct from a train/test split). The four metrics below are defined once on ${\tilde\nabla}$, and \emph{each is reported twice}: TracIn and KFAC-weighted.
\noindent Cosine metrics measure per-example gradient-direction fidelity before any cross-example influence score is formed:
\begin{equation*}
  {\cos_{\text{tr}} = \tfrac{1}{n_{\text{tr}}}\sum_{i=1}^{n_{\text{tr}}}\!\frac{\langle\tilde\nabla^{\text{d}}_{x_i},\tilde\nabla^{\text{c}}_{x_i}\rangle}{\|\tilde\nabla^{\text{d}}_{x_i}\|\,\|\tilde\nabla^{\text{c}}_{x_i}\|}},\qquad
  {\cos_{\text{te}} = \tfrac{1}{n_{\text{te}}}\sum_{j=1}^{n_{\text{te}}}\!\frac{\langle\tilde\nabla^{\text{d}}_{q_j},\tilde\nabla^{\text{c}}_{q_j}\rangle}{\|\tilde\nabla^{\text{d}}_{q_j}\|\,\|\tilde\nabla^{\text{c}}_{q_j}\|}}.
\end{equation*}
\noindent Query-ranking metrics measure whether each query induces the same influence ordering over source examples: Spearman uses the full ordering, while recall checks overlap among the top $5\%$. Let ${T_k(v)}$ denote the indices of the $k=\lceil0.05n_{\text{tr}}\rceil$ largest entries of $v$; then
\begin{equation*}
  {\rho_{\text{qry}}=\tfrac{1}{n_{\text{te}}}\sum_{j=1}^{n_{\text{te}}}
  \mathrm{Spearman}\!\big(\bfI^{\text{d}}_{\cdot j},\bfI^{\text{c}}_{\cdot j}\big)},\qquad
  {R_{\text{qry@5\%}}=\tfrac{1}{n_{\text{te}}}\sum_{j=1}^{n_{\text{te}}}
  \frac{\big|T_k(\bfI^{\text{d}}_{\cdot j})\cap T_k(\bfI^{\text{c}}_{\cdot j})\big|}{k}}.
\end{equation*}
\noindent \textbf{Baselines.} We group baselines by constraints imposed on compressed weights: (i) \textbf{Pure pruning}: \texttt{SparseGPT}, \texttt{ALPS}, and \texttt{LLM-Surgeon}; (ii) \textbf{Pure low-rank}: \texttt{SVD-LLM(W)}, \texttt{IPSVD}, and \texttt{IProX}; and (iii) \textbf{Sparse + low-rank \slr}: \texttt{OATS}, \texttt{3BASiL}. IPSVD is the closed-form solution of the weighted SVD in the $\bfL$-update, which is also the first stage of IProX. For \slr runs, we use rank $r_\ell = \lfloor \eta(1-c)\,m_\ell n_\ell/(m_\ell+n_\ell) \rfloor$ for target compression $c$, where $m_\ell$ and $n_\ell$ are the layer dimensions and $\eta$ is the fraction of retained parameters allocated to low rank; we use $\eta=0.3$, a value recommended by OATS. Additional implementation details are provided in \cref{sect:appendix-exp-details}.
\subsection{One-shot Compression Results}
We calibrate all methods on $128$ sequences of length $2048$ from C4's first shard and evaluate the four IF metrics on $n_{\text{tr}}=1000$ separate C4 training samples and $n_{\text{te}}=100$ WikiText-2 test samples, each of length $512$. \Cref{tab:main-sweep-appendix-tracin,tab:main-sweep,tab:main-sweep-appendix-l3b-ppl,tab:main-sweep-appendix-l8b-ppl} report results for Llama-3.2-1B, Qwen3-4B, Llama-3.2-3B, and Llama-3-8B, respectively. At $50\%$ compression on Llama-3.2-1B, \ourmethod reduces the C4 perplexity gap to dense by $28\%$ over the closest \slr method and $25\%$ over LLM-Surgeon(5sh), with \ourmethod-prune running ${>}15\times$ faster than the latter (\cref{tab:main-sweep-appendix-tracin}).
\begin{wraptable}[9]{r}{0.39\textwidth}
\centering
\small
\setlength{\tabcolsep}{5pt}
\begin{tabular}{lrr}
\toprule
Method & C4 $\downarrow$ & WT-2 $\downarrow$ \\
\midrule
\rowcolor{gray!10}
Dense & 11.27 & 7.05 \\
SparseGPT & 14.17 & 9.47 \\
\rowcolor{darkgreen!10}
\ourmethod-prune & \textbf{13.42} & \textbf{8.88} \\
\bottomrule
\end{tabular}
\caption{\footnotesize Qwen3.8-27B ppl at $2{:}4$ sparsity. 
}
\label{tab:qwen38-27b-ppl-main}
\end{wraptable}
It also retains substantially better downstream model quality than IProX's low-rank proxy, which achieves 102.61 WikiText-2 perplexity, against 14.17 for \ourmethod-prune. \Cref{fig:q4b-50pct} shows the C4 perplexity comparison on Qwen3-4B at $50\%$ compression. Under the hardware-friendly $(2{:}4{+}r256)$ configuration (1.32$\times$ inference speedup and $29\%$ memory reduction on Llama-3-8B with specialized kernels by~\cite{mozaffari2024slope}), \ourmethod-SLR improves per-query TracIn retention by $38\%$ over the closest competitor (\cref{fig:l8b-24-r256}). These results support important joint improvements in model quality and IF preservation; we next examine downstream data selection and integration with LoGRA projections.

\noindent\textbf{Discussion.} Better perplexity need not imply better gradient preservation. At $50\%$ Qwen3-4B compression, 3BASiL improves WikiText-2 perplexity over \ourmethod-prune ($15.34$ vs.\,$16.27$), but has lower dense-gradient cosine alignment ($0.743$ vs.\,$0.762$; \cref{tab:main-sweep-appendix-tracin-q4b}).

\begin{table}[t]
\centering
\scriptsize
\setlength{\tabcolsep}{3.2pt}
\newcolumntype{C}[1]{>{\centering\arraybackslash}p{#1}}
\resizebox{0.85\textwidth}{!}{%
\begin{tabular}{l|C{3.0em}C{3.0em}|C{3.6em}C{3.6em}C{3.6em}C{3.6em}|C{3.6em}C{3.6em}C{3.6em}C{3.6em}|C{4.2em}}
\toprule
& \cellcolor{gray!10}\textbf{C4}
& \cellcolor{gray!10}\textbf{WT-2}
& \multicolumn{4}{c|}{\textbf{TracIn} ($\uparrow$)}
& \multicolumn{4}{c|}{\textbf{KFAC} ($\uparrow$)}
& \multirow{2}{*}{\shortstack{\textbf{Compress}\\\textbf{Runtime}}} \\
\cmidrule(lr){4-7}\cmidrule(lr){8-11}
\textbf{Method}
& \cellcolor{gray!10}\textbf{PPL} ($\downarrow$)
& \cellcolor{gray!10}\textbf{PPL} ($\downarrow$)
& $\cos_{\text{tr}}$ & $\cos_{\text{te}}$ & $\rho_{\text{qry}}$ & $R_{\text{qry@5\%}}$
& $\cos_{\text{tr}}$ & $\cos_{\text{te}}$ & $\rho_{\text{qry}}$ & $R_{\text{qry@5\%}}$
& \\
\midrule
\rowcolor{gray!10} \textbf{Dense} & 19.85 & 13.66 & 1.000 & 1.000 & 1.000 & 1.000 & 1.000 & 1.000 & 1.000 & 1.000 & --- \\
\midrule[\heavyrulewidth]
\multicolumn{12}{c}{\cellcolor{blue!8}\textbf{30\% compression}} \\
\midrule
\multicolumn{12}{l}{\cellcolor{blue!4}\emph{Pure pruning}} \\
SparseGPT          & 21.05 & 14.53 & 0.852 & 0.862 & 0.959 & 0.746 & 0.739 & 0.750 & 0.498 & 0.292 & 24m 53s \\
ALPS               & 20.44 & \textbf{13.84} & 0.892 & 0.903 & 0.977 & 0.820 & 0.794 & 0.805 & 0.595 & 0.375 & 43m 07s \\
LLM-Surgeon(5sh)   & 20.44 & 14.33 & 0.910 & 0.915 & 0.971 & 0.829 & \textbf{0.825} & \textbf{0.836} & \textbf{0.630} & 0.373 & 18h 41m$^{\dagger}$ \\
\rowcolor{darkgreen!10}
\texttt{CACTUS-prune}
                   & \textbf{20.21}
                   & 13.87
                   & \textbf{0.920}
                   & \textbf{0.923}
                   & \textbf{0.984}
                   & \textbf{0.837}
                   & 0.816
                   & 0.816
                   & 0.624
                   & \textbf{0.388} & 45m 59s \\
\midrule
\multicolumn{12}{l}{\cellcolor{blue!4}\emph{Pure low-rank}} \\
SVD-LLM(W)         & 47.31 & 70.50 & 0.178 & 0.183 & 0.309 & 0.103 & 0.075 & 0.079 & 0.018 & 0.075 & 54m 43s \\
IPSVD              & \textbf{37.79} & \textbf{60.98} & \textbf{0.378} & \textbf{0.383} & \textbf{0.545} & \textbf{0.230} & \textbf{0.141} & \textbf{0.140} & \textbf{0.036} & \textbf{0.081} & 51m 00s \\
\midrule
\multicolumn{12}{l}{\cellcolor{blue!4} \slr} \\
OATS               & 20.13 & 13.96 & 0.892 & 0.889 & 0.969 & 0.748 & 0.790 & 0.788 & 0.588 & 0.360 & 2h 13m \\
3BASiL             & 20.15 & \underline{\textbf{13.72}} & 0.922 & 0.928 & 0.980 & 0.819 & 0.857 & 0.858 & 0.698 & 0.441 & 3h 03m \\
\rowcolor{darkgreen!10}
\texttt{CACTUS}-SLR
                   & \underline{\textbf{20.06}}
                   & 13.73
                   & \underline{\textbf{0.937}}
                   & \underline{\textbf{0.945}}
                   & \underline{\textbf{0.987}}
                   & \underline{\textbf{0.859}}
                   & \underline{\textbf{0.866}}
                   & \underline{\textbf{0.871}}
                   & \underline{\textbf{0.715}}
                   & \underline{\textbf{0.460}} & 3h 50m \\
\midrule[\heavyrulewidth]
\multicolumn{12}{c}{\cellcolor{blue!8}\textbf{50\% compression}} \\
\midrule
\multicolumn{12}{l}{\cellcolor{blue!4}\emph{Pure pruning}} \\
SparseGPT          & 24.43 & 18.04 & 0.656 & 0.660 & 0.874 & 0.551 & 0.387 & 0.388 & 0.138 & 0.128 & 24m 07s \\
ALPS               & 22.60 & 16.40 & 0.709 & 0.710 & 0.895 & 0.585 & 0.458 & 0.455 & 0.190 & \textbf{0.150} & 43m 24s \\
LLM-Surgeon(5sh)   & 23.50 & 19.80 & 0.680 & 0.670 & 0.864 & 0.561 & 0.397 & 0.383 & 0.140 & 0.125 & 20h 13m$^{\dagger}$ \\
\rowcolor{darkgreen!10}
\texttt{CACTUS-prune}
                   & \textbf{21.85}
                   & \textbf{16.27}
                   & \textbf{0.762}
                   & \textbf{0.755}
                   & \textbf{0.927}
                   & \textbf{0.697}
                   & \textbf{0.493}
                   & \textbf{0.474}
                   & \textbf{0.211}
                   & 0.149 & 42m 18s \\
\midrule
\multicolumn{12}{l}{\cellcolor{blue!4}\emph{Pure low-rank}} \\
SVD-LLM(W)         & 176.58 & 213.73 & 0.113 & 0.118 & 0.193 & 0.052 & 0.007 & 0.009 & 0.013 & 0.047 & 53m 15s \\
IPSVD              & \textbf{94.19} & \textbf{167.49} & \textbf{0.146} & \textbf{0.151} & \textbf{0.354} & \textbf{0.143} & \textbf{0.030} & \textbf{0.032} & \textbf{0.021} & \textbf{0.062} & 53m 27s \\
\midrule
\multicolumn{12}{l}{\cellcolor{blue!4} \slr} \\
OATS               & 23.97 & 16.73 & 0.652 & 0.653 & 0.804 & 0.508 & 0.380 & 0.383 & 0.140 & 0.130 & 2h 05m \\
3BASiL             & 21.63 & \underline{\textbf{15.34}} & 0.743 & 0.747 & 0.919 & 0.695 & 0.521 & 0.521 & 0.252 & \underline{\textbf{0.187}} & 2h 37m \\
\rowcolor{darkgreen!10}
\texttt{CACTUS}-SLR
                   & \underline{\textbf{21.38}}
                   & 15.59
                   & \underline{\textbf{0.782}}
                   & \underline{\textbf{0.790}}
                   & \underline{\textbf{0.937}}
                   & \underline{\textbf{0.698}}
                   & \underline{\textbf{0.539}}
                   & \underline{\textbf{0.538}}
                   & \underline{\textbf{0.259}}
                   & 0.173 & 2h 54m \\
\bottomrule
\end{tabular}%
}
\vspace{3pt}
\caption{One-shot compression of Qwen3-4B at $30\%$ and $50\%$ compression ratio. \textbf{Bold} = best within its constraint set; \underline{\textbf{underline}} = best overall across all reported methods. Runtime results are reported using an L40 GPU; $^{\dagger}$ denotes LLM-Surgeon runtimes on two NVIDIA A100 GPUs. More results are deferred to Appendix~\ref{app:if-ci}.}
\label{tab:main-sweep}
\end{table}
\begin{figure}[t]
\centering
\begin{minipage}[t]{0.49\textwidth}
\centering
\includegraphics[width=\textwidth]{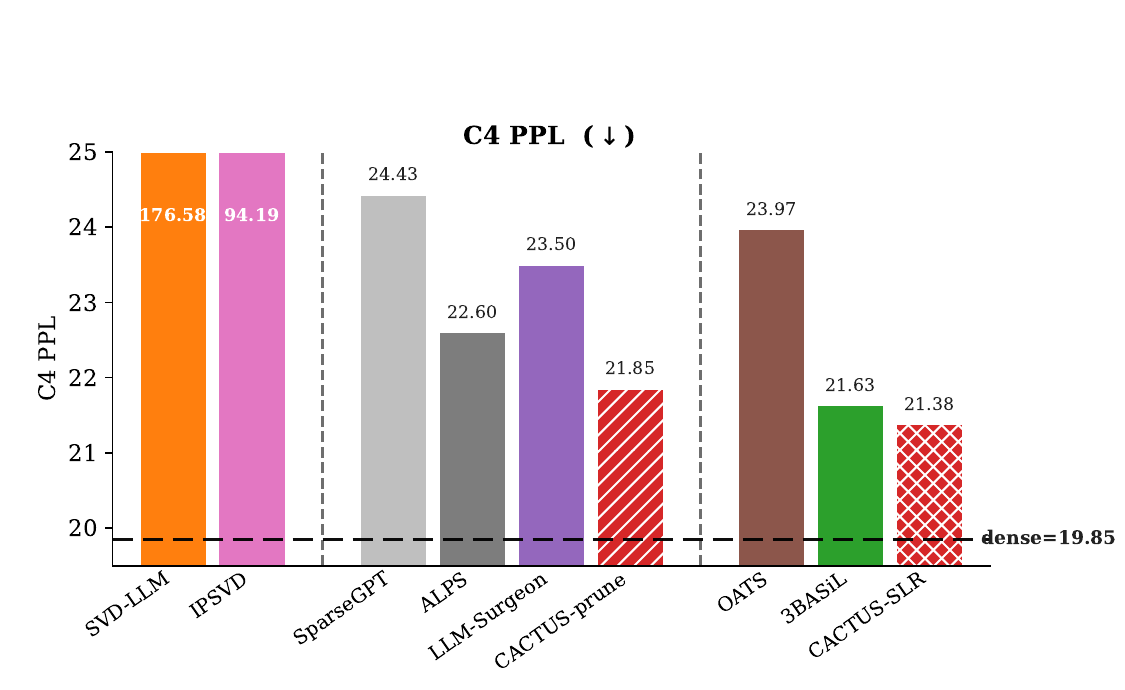}
\caption{\small Qwen3-4B at matched $50\%$ compression.}
\label{fig:q4b-50pct}
\end{minipage}
\hfill
\begin{minipage}[t]{0.49\textwidth}
\centering
\includegraphics[width=\textwidth]{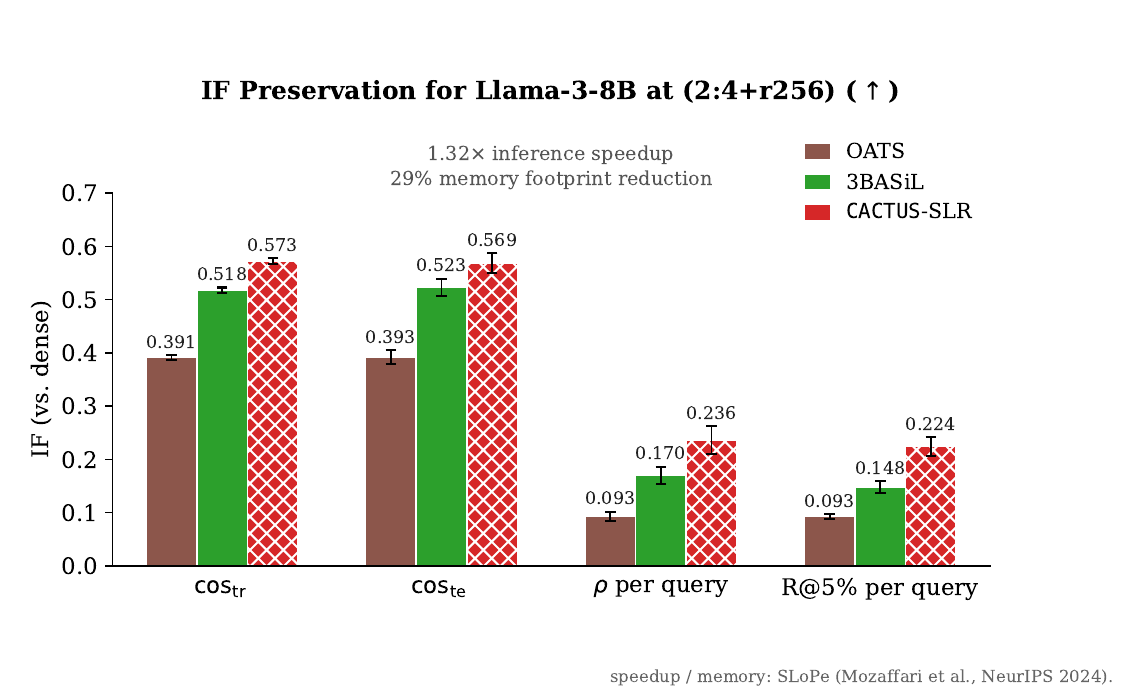}
\caption{\small Llama-3-8B at $(2{:}4{+}r256)$.}
\label{fig:l8b-24-r256}
\end{minipage}
\end{figure}
\subsection{Data Selection Results}
\label{sect:data-selection}
\noindent Having established the large gains of \ourmethod in terms of efficiency compared to LLM-Surgeon as well as utility and IF preservation compared to layerwise reconstruction pruning methods (e.g., SparseGPT), we now evaluate whether IF preservation also benefits downstream data selection using the LESS~\citep{xia2024less} pipeline on BIG-Bench Hard (BBH)~\citep{suzgun2023challenging}.
For Llama-2-7B and Llama-3.2-3B, we compare the pruning methods SparseGPT, ALPS, and \ourmethod-prune at matched layerwise budgets, using dense-model and random selection as references. Each scoring model uses the same LESS procedure, with optimizer-aware gradients and TRAK projections, to select $5\%$ of the $270,679$ example LESS instruction-tuning mixture comprising FLAN-v2, CoT, Dolly, and OpenAssistant. 
\begin{wraptable}[12]{r}{0.48\textwidth}
\centering
\setlength{\tabcolsep}{3pt}
\resizebox{\linewidth}{!}{%
\begin{tabular}{l ccc ccc}
\toprule
& \multicolumn{3}{c}{Llama-2-7B} & \multicolumn{3}{c}{Llama-3.2-3B} \\
\cmidrule(lr){2-4}\cmidrule(lr){5-7}
Selector & $\rho_{\mathrm{sel}}\uparrow$ & Overlap $\uparrow$ & EM $\uparrow$ & $\rho_{\mathrm{sel}}\uparrow$ & Overlap $\uparrow$ & EM $\uparrow$ \\
\midrule
Dense LESS & $1.000$ & $100.00$ & $42.59$ & $1.000$ & $100.00$ & $49.54$ \\
\midrule[0.3pt]
\rowcolor{darkgreen!10}
\ourmethod-prune & $\mathbf{0.638}$ & $\mathbf{47.20}$ & $41.57$ & $\mathbf{0.678}$ & $\mathbf{54.13}$ & $\mathbf{49.44}$ \\
ALPS & $0.594$ & $44.07$ & $40.46$ & $0.621$ & $52.24$ & $47.22$ \\
SparseGPT & $0.587$ & $39.70$ & $\mathbf{41.67}$ & $0.598$ & $49.60$ & $46.02$ \\
\midrule
Random $5\%$ & --- & $4.45$ & $39.81$ & --- & $4.94$ & $48.89$ \\
\bottomrule
\end{tabular}}
\caption{\footnotesize LESS selection. $\rho_{\mathrm{sel}}$: Spearman correlation of aggregate scores. Overlap: selected examples shared with dense selected data. EM: BBH acc (\%).}
\label{tab:less-bbh-selection}
\end{wraptable}
We then fine-tune a fresh dense model on each subset using the same training recipe as LESS and evaluate exact match (EM) on $1080$ held-out examples across $27$ BBH tasks with 3-shot chain-of-thought (CoT) prompting. Additional setup details are deferred to \cref{sect:appendix-exp-details}. \ourmethod achieves the highest score-rank correlation $0.638$ against the closest competitor $0.594$ and selected-set overlap with dense-selection at $47.20\%$ compared with $44.07\%$ for the better competitor. It also achieves the best BBH EM on Llama-3.2-3B $49.44$ compared with $47.22$ for ALPS.

\vspace{-0.4em}
\subsection{LoGRA Integration}
\vspace{-0.4em}
\noindent In the previous subsections, we examined the performance of \ourmethod using TRAK to sketch full per-sample gradients to $d=8192$ dimensions. We now examine its compatibility with another gradient projection strategy: Low-rank Gradient Approximation (LoGRA)~\citep{choe2025whatlogra}. LoGRA projects layerwise activations and backpropagated gradients without materializing the full per-sample gradient. It attaches a random rank-$r$ probe $(\bfP_o,\bfP_i)$ to each tracked linear layer and records a projected per-sample gradient ${\bm\Phi_\ell(x){=}(\bfP_o^{\!\top}g_{x,\ell})(\bfP_i x_\ell)^{\!\top}\in\R^{r\times r}}$, where $x_\ell$ and $g_{x,\ell}$ are the layer input and backpropagated gradient for sample $x$, and $\bfP_i$ and $\bfP_o$ are the fixed input and output projection matrices. Recall that the full per-sample gradient of the layer is given by $g_{x,\ell} x_\ell^{\!\top}$.

We treat Dense+LoGRA influence scores as the oracle and compare compressed models (also using the same LoGRA projection probes) against this reference. All rows use the same rank-$64$ random LoGRA factors on the MLP linears (default LoGRA setting). The \textbf{(middle)} panel shows the $2{:}4$ pruning-only case, where the layer is compressed and then we append the same rank-$r$ probe as the dense version of the \textbf{(left)} panel. 
For sparse-plus-low-rank models, we incorporate the low-rank correction into the LoGRA adapter while preserving exactly the same gradient projection used for the dense and pruned models. This allows a direct comparison under identical projection settings.
\begin{figure}[H]
\vspace{-0.3em}
\centering
\def\LograMatrixFont{\fontsize{22}{26}\selectfont} 
\def\LograLabelFont{\fontsize{12}{14}\selectfont}  
\centering
\resizebox{0.90\textwidth}{!}{%
\begin{tikzpicture}[
  >=Stealth,
  line cap=round,
  line join=round,
  font=\sffamily,
  flow/.style={-{Stealth[length=2.8mm,width=2.2mm]}, line width=0.78pt},
  route/.style={flow, rounded corners=3pt},
  module/.style={
    draw=moduleborder, fill=moduleblue, rounded corners=2.5pt,
    line width=0.55pt, minimum width=3.45cm, minimum height=4.45cm,
    align=center
  },
  projector/.style={draw=opborder, fill=opfill, line width=0.55pt},
  zero/.style={
    draw=zeroborder, fill=zerofill, rounded corners=2pt,
    line width=0.55pt, minimum width=1.18cm, minimum height=1.18cm,
    align=center
  },
  core/.style={
    draw=zeroborder, rounded corners=1.5pt, line width=0.55pt,
    minimum width=0.82cm, minimum height=0.74cm, align=center
  },
  groupbox/.style={draw=black, fill=none, rounded corners=3pt, line width=0.65pt},
  caption/.style={font=\LograLabelFont\ttfamily\bfseries, align=center},
  tinylabel/.style={font=\LograLabelFont\bfseries}
]

\definecolor{moduleblue}{HTML}{DEE8FB}
\definecolor{moduleborder}{HTML}{93A9CF}
\definecolor{opfill}{HTML}{BFE8EB}
\definecolor{opborder}{HTML}{4B91A0}
\definecolor{zerofill}{HTML}{E0D7E6}
\definecolor{zeroborder}{HTML}{9A77AF}
\definecolor{flowgreen}{HTML}{4F7D25}
\definecolor{flowred}{HTML}{C73A27}
\definecolor{flowpurple}{HTML}{9641F8}

\def\LograPanel#1#2#3#4{%
\begin{scope}[shift={(#1,0)}]
  \node[module] (M) at (1.72,3.35) {\LograMatrixFont #2};

  \path[projector]
    (4.02,5.58) -- (7.32,5.58) -- (6.18,4.76) -- (5.16,4.76) -- cycle;
  \node at (5.67,5.15) {\LograLabelFont $\mathbf{P_o}$};

  \path[projector]
    (5.16,1.94) -- (6.18,1.94) -- (7.32,1.12) -- (4.02,1.12) -- cycle;
  \node at (5.67,1.50) {\LograLabelFont $\mathbf{P_i}$};

  \node[zero] (Z) at (5.67,3.35) {\LograLabelFont $\mathbf{0}$};

  \node[tinylabel, flowgreen] at (1.30,7.02) {$x_o$};
  \node[tinylabel, flowred] at (1.86,7.02) {$D x_o$};
  \node[tinylabel, flowgreen] at (1.31,-0.32) {$x_i$};
  \node[tinylabel, flowred] at (1.87,-0.32) {$D x_i$};

  \draw[flow, flowgreen] (1.31,5.60) -- (1.31,6.76);
  \draw[flow, flowred]   (1.82,6.76) -- (1.82,5.60);
  \draw[flow, flowgreen] (1.31,-0.10) -- (1.31,1.08);
  \draw[flow, flowred]   (1.82,1.08) -- (1.82,-0.10);

  \draw[route, flowgreen]
    (5.67,5.62) -- (5.67,6.12) -- (1.31,6.12);
  \draw[route, flowred]
    (1.82,6.34) -- (6.08,6.34) -- (6.08,5.58);
  \draw[route, flowgreen]
    (1.31,0.46) -- (5.67,0.46) -- (5.67,1.12);
  \draw[route, flowred]
    (6.08,1.12) -- (6.08,0.26) -- (1.82,0.26);

  \draw[flow, flowgreen] (5.50,3.94) -- (5.50,4.76);
  \draw[flow, flowred]   (5.92,4.76) -- (5.92,3.94);
  \draw[flow, flowgreen] (5.50,1.94) -- (5.50,2.76);
  \draw[flow, flowred]   (5.92,2.76) -- (5.92,1.94);

  \node[tinylabel, flowpurple, anchor=center] at (3.35,7.55) {#3};

  \node[caption] at (3.35,-1.10) {#4};
\end{scope}
}

\def\LograSLRPanel#1#2#3#4{%
\begin{scope}[shift={(#1,0)}]
  \node[module] (M) at (1.72,3.35) {\LograMatrixFont #2};

  \path[projector, draw=moduleborder, fill=moduleblue]
    (4.02,5.58) -- (6.42,5.58) -- (5.60,4.86) -- (4.84,4.86) -- cycle;
  \node at (5.22,5.22) {\LograLabelFont $\mathbf{V^\star}$};

  \path[projector]
    (6.68,5.58) -- (9.08,5.58) -- (8.26,4.86) -- (7.50,4.86) -- cycle;
  \node at (7.88,5.22) {\LograLabelFont $\mathbf{P_o}$};

  \node[zero, draw=moduleborder, fill=moduleblue] at (5.22,3.35)
    {\LograLabelFont $\mathbf{\Sigma^\star}$};
  \node[zero] at (7.88,3.35) {\LograLabelFont $\mathbf{0}$};

  \path[projector, draw=moduleborder, fill=moduleblue]
    (4.84,1.84) -- (5.60,1.84) -- (6.42,1.12) -- (4.02,1.12) -- cycle;
  \node at (5.22,1.48) {\LograLabelFont $\mathbf{U^\star}$};

  \path[projector]
    (7.50,1.84) -- (8.26,1.84) -- (9.08,1.12) -- (6.68,1.12) -- cycle;
  \node at (7.88,1.48) {\LograLabelFont $\mathbf{P_i}$};

  \draw[groupbox] (3.88,4.70) rectangle (9.22,5.74);
  \draw[groupbox] (4.42,2.64) rectangle (8.68,4.06);
  \draw[groupbox] (3.88,0.98) rectangle (9.22,2.00);

  \node[tinylabel, flowgreen] at (1.30,7.02) {$x_o$};
  \node[tinylabel, flowred] at (1.86,7.02) {$D x_o$};
  \node[tinylabel, flowgreen] at (1.31,-0.32) {$x_i$};
  \node[tinylabel, flowred] at (1.87,-0.32) {$D x_i$};

  \draw[flow, flowgreen] (1.31,5.60) -- (1.31,6.76);
  \draw[flow, flowred]   (1.82,6.76) -- (1.82,5.60);
  \draw[flow, flowgreen] (1.31,-0.10) -- (1.31,1.08);
  \draw[flow, flowred]   (1.82,1.08) -- (1.82,-0.10);

  \draw[route, flowgreen]
    (5.22,5.58) -- (5.22,6.12) -- (1.31,6.12);
  \draw[route, flowgreen]
    (7.88,5.58) -- (7.88,6.12) -- (1.31,6.12);
  \draw[route, flowred]
    (1.82,6.48) -- (5.62,6.48) -- (5.62,5.58);
  \draw[route, flowred]
    (1.82,6.48) -- (8.28,6.48) -- (8.28,5.58);
  \draw[route, flowgreen]
    (1.31,0.46) -- (5.22,0.46) -- (5.22,1.12);
  \draw[route, flowgreen]
    (1.31,0.46) -- (7.88,0.46) -- (7.88,1.12);
  \draw[route, flowred]
    (5.62,1.12) -- (5.62,0.24) -- (1.82,0.24);
  \draw[route, flowred]
    (8.28,1.12) -- (8.28,0.24) -- (1.82,0.24);

  \draw[flow, flowgreen] (5.05,1.84) -- (5.05,2.76);
  \draw[flow, flowgreen] (5.05,3.94) -- (5.05,4.86);
  \draw[flow, flowgreen] (7.71,1.84) -- (7.71,2.76);
  \draw[flow, flowgreen] (7.71,3.94) -- (7.71,4.86);

  \draw[flow, flowred] (5.39,4.86) -- (5.39,3.94);
  \draw[flow, flowred] (5.39,2.76) -- (5.39,1.84);
  \draw[flow, flowred] (8.05,4.86) -- (8.05,3.94);
  \draw[flow, flowred] (8.05,2.76) -- (8.05,1.84);

  \node[tinylabel, flowpurple, anchor=center] at (5.35,7.55) {#3};

  \node[caption] at (3.70,-1.10) {#4};
\end{scope}
}

\LograPanel{0.45}
  {$\mathbf{W}$}
  {$\mathbf{P_g(W)}$}
  {\huge Standard LoGRA+Dense}

\LograPanel{10.55}
  {$\mathbf{W_{2:4}}$}
  {$\mathbf{P_g(W_{2:4})}$}
  {\huge LoGRA+\ourmethod-prune}

\LograSLRPanel{19.60}
  {$\mathbf{W_{2:4}}$}
  {$\mathbf{P_g(W_{2:4}+L^\star)}$}
  {\huge \hspace{2em}LoGRA+\ourmethod-SLR}

\draw[-{Stealth[length=4.2mm,width=3.2mm]}, line width=0.95pt]
  (8.18,3.35) -- (10.15,3.35);
\draw[-{Stealth[length=4.2mm,width=3.2mm]}, line width=0.95pt]
  (18.28,3.35) -- (19.25,3.35);
\end{tikzpicture}%
}
\end{figure}

The \textbf{(right)} panel illustrates this construction for \ourmethod-SLR, whose deployed weight is $\bfS^\star{+}\bfL^\star$. We decompose $\bfL^\star$ using SVD and concatenate its factors with the random LoGRA probes, widening the adapter to rank $2r$. The augmented branch contributes exactly $\bfL^\star$ to the forward pass, so the complete layer still computes with $\bfS^\star{+}\bfL^\star$; the random probe contributes zero to the output. After backpropagation, we obtain the per-sample gradient with respect to the adapter's central $2r\times2r$ matrix and retain only the $r\times r$ block associated with the original random LoGRA probe. 
\begin{wrapfigure}{r}{0.55\textwidth}
\centering
\resizebox{0.55\textwidth}{!}{%
\scriptsize
\setlength{\tabcolsep}{3pt}
\begin{tabular}{l|ccccc}
\toprule
\textbf{Method MLP-Layers}
 & $\cos_{\text{tr}}$ & $\cos_{\text{te}}$ & $\rho_{\text{qry}}$ & $R_{\text{qry}@5\%}$ & \textbf{Compression} \\
\midrule
\rowcolor{gray!10}\textbf{LoGRA\,+\,Dense (oracle)}                             & $1.000$ & $1.000$ & $1.000$ & $1.000$ & --- \\
LoGRA\,+\,SparseGPT       & $0.593$ & $0.489$ & $0.628$ & $0.398$ & 19m\,47s \\
LoGRA\,+\,LLM-Surgeon(1sh) & $0.630$ & $0.557$ & $0.708$ & $0.445$ & 4h\,02m$^{\dagger}$ \\
LoGRA\,+\,LLM-Surgeon(5sh) & $0.635$ & $0.570$ & $0.706$ & $0.499$ & 19h\,55m$^{\dagger}$ \\
\rowcolor{darkgreen!10}
LoGRA\,+\,\ourmethod-prune & $0.645$ & $0.584$ & $0.730$ & $0.583$ & 49m\,36s \\
\rowcolor{darkgreen!10}
LoGRA\,+\,\ourmethod-SLR & $\mathbf{0.711}$ & $\mathbf{0.639}$ & $\mathbf{0.804}$ & $\mathbf{0.597}$ & 1h\,33m \\
\bottomrule
\end{tabular}%
}
\caption{Influence Function preservation metrics for Qwen3-4B with MLP layers pruned using 2:4 sparsity. $^{\dagger}$ denotes runtimes that use 2 A100 GPUs; other reported compression runtimes use a single L40 GPU.}
\label{tab:logra-concat-q4b}
\end{wrapfigure}
This block is exactly the standard rank-$r$ LoGRA gradient evaluated on the deployed \slr model. On Qwen3-4B, \ourmethod-SLR improves all LoGRA-projected IF preservation metrics over pruning procedures, reaching a per-query $0.804$ Spearman and $0.597$ recall at $5\%$. \ourmethod-prune is competitive and uses a single L40 GPU against 2 A100 GPUs for LLM-Surgeon$^{\dagger}$ while running roughly $5\times$ and $25\times$ faster than the 1-shot and 5-shot variants, respectively.

\vspace{-0.4em}
\section{Related Work}
\vspace{-0.4em}
\label{sec:related-work}
\noindent\textbf{One-shot pruning and low-rank decompositions for LLMs.}
Post-training one-shot compression is a training-free paradigm popularized by SparseGPT~\citep{frantar2023sparsegpt}, which compresses a pretrained LLM layer-by-layer on a small calibration set by minimizing the layerwise reconstruction error $\|\mathbf{X}\mathbf{W}-\mathbf{X}\widehat{\mathbf{W}}\|_F^2$ subject to a structural constraint on $\widehat{\mathbf{W}}$; methods differ along two axes: the constraint family and the per-layer solver.
\emph{Pure sparsity}: SparseGPT~\citep{frantar2023sparsegpt} uses a sequential column-based sparsification scheme: after sparsifying one column, the remaining weights are updated to compensate for the incurred error, while ALPS~\citep{meng2024alps} formulates the same problem as a constrained quadratic problem solved by the Alternating Direction Method of Multipliers (ADMM)~\citep{boyd2011distributed}.
\emph{Pure low-rank}: SVD-LLM(W)~\citep{wang2025svdllm} uses a whitened Singular Value Decomposition (SVD) to minimize layerwise reconstruction error; ASVD~\citep{yuan2025asvd} instead scales input channels using activation statistics before SVD.
\emph{Sparse + low-rank} \slr: OATS~\citep{zhang2024oats} and Hassle-free~\citep{makni2025hasslefree} alternate a sparse fitting step with a low-rank projection step; OATS scores weights using a per-row activation-norm criterion whereas Hassle-free uses the full input correlation matrix. 3BASiL~\citep{makni20263basil} runs a multi-block ADMM that jointly minimizes the sparse and low-rank components simultaneously. Our method \ourmethod differs from all these layerwise reconstruction algorithms by incorporating downstream gradients in its objective, making it curvature-aware rather than \emph{gradient-blind}. This is best illustrated in \cref{fig:curvature-matters}.

\noindent\textbf{Curvature-Aware Compression.}
A complementary thread incorporates downstream sensitivity into model compression. LLM-Surgeon~\citep{ouderaa2024llmsurgeon} and IProX~\citep{chen2026influencepreservingiprox} arrive at the curvature-aware objective $\|\mathbf{X}(\mathbf{W}-\widehat{\mathbf{W}})\mathbf{A}^\top\|_F^2$ from different motivations: LLM-Surgeon approximates the change in model loss, whereas IProX seeks to preserve gradient information for influence estimation (\cref{app:derivation-diagram,fig:derivation-convergence}). LLM-Surgeon targets pure pruning through curvature-based weight selection and correlated corrections over multiple rounds; IProX constructs a pure low-rank proxy through a weighted SVD followed by gradient-alignment fine-tuning. YAQA~\citep{tseng2026modelpreserving} is a related approach which applies downstream curvature to weight-only quantization. Its curvature-based error feedback is conceptually closer to LLM-Surgeon's error compensation than to our ADMM splitting procedure. \ourmethod differs from these prior methods as it optimizes the curvature-aware loss for pruning and \slr using multi-block ADMM with efficient subroutines: a closed-form Sylvester solve in cached eigenbases, sparsity projections, and a randomized-SVD-based low-rank approximation (\cref{sec:algorithm}). On Qwen3-4B at $50\%$ compression, \ourmethod-prune is over $28\times$ faster than five-shot LLM-Surgeon$^{\dagger}$, using one L40 rather than two A100 GPUs (\cref{tab:main-sweep-appendix-tracin-q4b,sect:appendix-exp-details}). Our pruning and \slr representations also achieve substantially lower perplexity than IProX's aligned low-rank proxy (\cref{tab:main-sweep-appendix-tracin}).

\noindent\textbf{Efficient Influence Function Estimation.}
Influence functions, rooted in robust statistics \citep{hampel1974influence} and introduced to deep learning by \citet{koh2017understanding}, estimate a training example's effect via $\mathcal{I}(z,z') = \nabla\ell(z')^\top \mathbf{H}^{-1} \nabla\ell(z)$, supporting data valuation, instruction tuning, and attribution \citep{tracin,xia2024less,choe2025whatlogra,grosse2023studying,chen2026influencepreservingiprox}.
At LLM scale, both the per-sample gradient $\nabla\ell(z)$ and the inverse-Hessian-vector product $\mathbf{H}^{-1}\nabla\ell(z)$ are prohibitively expensive to compute and store exactly. We focus on two broadly complementary directions that make influence estimation tractable.
\textbf{(i) Compressed-model proxies} compute influence on smaller surrogates designed to preserve the target model's gradients. IProX~\citep{chen2026influencepreservingiprox} uses an influence-preserving SVD followed by gradient-alignment fine-tuning. \ourmethod extends compressed-model proxies to pruning and \slr through the same curvature-aware objective.
\textbf{(ii) Gradient projection} maps per-sample gradients to compact vectors for efficient similarity computation. TRAK~\citep{park2023trak} uses a \emph{data-free} Johnson Lindenstrauss (JL) sketch, while Arnoldi-IF~\citep{schioppa2022scaling} uses a \emph{data-dependent} approximation to the leading Hessian eigenspace. Low-rank Gradient Approximation (LoGRA)~\citep{choe2025whatlogra} exploits the per-layer gradient structure $\operatorname{vec}(\mathcal{D}\mathbf{W}) = \sum_t x_{i,t}\otimes\mathcal{D}x_{o,t}$ through a matching projection $P = P_i \otimes P_o$, implemented as small low-rank adaptation (LoRA) add-on layers.
Our work primarily contributes to direction \textbf{(i)}: where IProX is restricted to pure low-rank, \ourmethod delivers pruning and \slr proxies with improved utility-compression tradeoffs. We also evaluate compatibility with both TRAK's random projections and LoGRA's architecture-aware projections from direction \textbf{(ii)}.

\vspace{-0.4em}
\section{Conclusion and Limitations}
\vspace{-0.4em}
\label{sec:conclusion}

In this paper, we first highlight the connection between downstream utility and influence preservation through the curvature-aware objective introduced in \cref{eq:kfac-loss}. Our new algorithm \ourmethod makes optimizing this objective practical for modern LLM pruning as well as \slr decomposition tasks. It runs at a substantially lower cost compared to prior state-of-the-art curvature-aware pruning methods. 
Across various compression settings, model architectures, and gradient projection strategies, \ourmethod consistently improves downstream utility while retaining gradient fidelity.

\textbf{Limitations.} In this paper, we focus on pruning and sparse-plus-low-rank \slr compression. Extending our ADMM framework to quantization and mixed quantization-sparsity settings are directions for future work, especially when reparameterizations (e.g., rotations) change gradient geometry. We also only consider uniform compression budgets; it would be interesting to design specialized methods that can allocate sparsity and rank non-uniformly across the entire LLM layers. Additionally, we do not study post-compression refinements like full-model or LoRA-style finetuning. These may trade off deployment quality and gradient alignment fidelity in different ways. We leave these interesting follow-up research directions for future work.

\subsection*{AI use statement}
Generative AI assisted with code development, manuscript revision and polishing, and figure plotting and formatting. The authors are responsible for the final content and verification of the reported results.


\bibliographystyle{plainnat}
\bibliography{ref}

\clearpage
\newpage

\appendix

\section{The LLM-Surgeon and IProX Derivations}\label{app:derivation-diagram}

\Cref{fig:derivation-convergence} connects the two motivations for \cref{eq:kfac-loss}. For a layer $y=W^\top x$, let $\Delta W=\bfWhat-\bfW$ and $g=\nabla_y\ell(z)$, where $z$ is a sample and $\ell$ its loss. The input $x$ and gradient $g$ are evaluated at the dense model; $\mathcal L=\mathbb E_z[\ell(z)]$.

\begin{figure}[h!]
\centering
\resizebox{\textwidth}{!}{%
\begin{tikzpicture}[
    >=Latex,
    every node/.style       = {inner sep=0pt, outer sep=0pt},
    stepL/.style            = {rounded corners=5pt, draw=datablue!65!black,
                               line width=0.55pt, fill=databluebg!55,
                               align=center, inner sep=6pt,
                               font=\footnotesize, minimum width=5.2cm,
                               minimum height=1.1cm},
    stepR/.style            = {rounded corners=5pt, draw=ifteal!60!black,
                               line width=0.55pt, fill=iftealbg!55,
                               align=center, inner sep=6pt,
                               font=\footnotesize, minimum width=5.2cm,
                               minimum height=1.1cm},
    gateWarn/.style         = {rounded corners=3pt, draw=warnred!75!black,
                               line width=0.7pt, fill=warnredbg,
                               align=center, inner sep=4pt,
                               font=\footnotesize,
                               text=warnred!20!black,
                               minimum width=3.4cm},
    arr/.style              = {-{Latex[length=4pt, width=3.4pt]},
                               line width=0.9pt, draw=black!55,
                               line cap=round},
    arrL/.style             = {-{Latex[length=4pt, width=3.4pt]},
                               line width=1pt, draw=datablue!60!black,
                               line cap=round},
    arrR/.style             = {-{Latex[length=4pt, width=3.4pt]},
                               line width=1pt, draw=ifteal!55!black,
                               line cap=round},
    kmerge/.style           = {rounded corners=6pt, draw=black!55,
                               line width=0.85pt, fill=black!5,
                               align=center, inner sep=7pt,
                               font=\footnotesize, text=black!80,
                               minimum width=7.6cm, minimum height=1.6cm},
    share/.style            = {rounded corners=7pt, draw=kfacpurple!70!black,
                               line width=1pt, fill=kfacpurple!6,
                               align=center, inner sep=8pt}
]

\providecolor{datablue}{RGB}{40,110,190}
\providecolor{databluebg}{RGB}{230,242,252}
\providecolor{kfacpurple}{RGB}{120,60,170}
\definecolor{ifteal}{RGB}{20,140,125}
\definecolor{iftealbg}{RGB}{225,245,240}
\definecolor{warnred}{RGB}{200,40,45}
\definecolor{warnredbg}{RGB}{253,230,232}

\node[anchor=center, font=\normalsize, align=center] at (7.6, 10.90)
  {\textbf{One objective, two derivations}};

\node[anchor=center, font=\small\bfseries, text=datablue!25!black]
   at (2.6, 10.2) {LLM-Surgeon\, $\rhd$\, \textsc{model compression}};
\node[anchor=center, font=\small\bfseries, text=ifteal!25!black]
   at (12.6, 10.2){IProX (IPSVD stage)\, $\rhd$\, \textsc{influence preservation}};

\node[stepL] (L1) at (2.6, 9.10)
   {\textbf{Second-order Taylor Expansion}\\[2pt]
    Near a stationary point\\[2pt]
    $\displaystyle \Delta \mathcal L \;\approx\; \tfrac12\,
        (\Delta\theta)^{\!\top}\, \nabla_\theta^2\mathcal L\, (\Delta\theta)$};

\node[stepL] (L2) at (2.6, 6.70)
   {\textbf{Fisher substitution}\\[2pt]
    $\displaystyle \nabla_\theta^2\mathcal L \;\approx\; F \;=\; \mathbb{E}\!\left[
        \nabla_\theta\ell\,(\nabla_\theta\ell)^{\!\top}\right]$};

\node[stepL, minimum width=4.0cm, minimum height=1.5cm] (L3) at (2.6, 4.30)
   {\textbf{Block-diagonal approximation (per layer)}\\[4pt]
    \mbox{\begin{tikzpicture}[baseline=(bdm.base), scale=0.32]
       \node[inner sep=0pt] (bdm) at (2.5, 2.5) {};
       \draw[black!55, line width=0.45pt, rounded corners=1pt]
          (0,0) rectangle (5,5);
       \foreach \i in {0,1,2,3,4} {
         \pgfmathsetmacro{\y}{4-\i}
         \fill[datablue!35!white, draw=datablue!75!black, line width=0.4pt,
               rounded corners=0.6pt]
            (\i+0.15, \y+0.15) rectangle (\i+0.85, \y+0.85);
       }
     \end{tikzpicture}}\\[4pt]
    \scriptsize $F \;\approx\; \operatorname{blkdiag}(F_1, \ldots, F_L)$};

\draw[arrL] (L1.south) -- (L2.north);
\draw[arrL] (L2.south) -- (L3.north);

\node[gateWarn] at (2.6, 7.85)
   {Empirical Fisher surrogate};

\node[stepR, minimum width=6.8cm] (R1) at (12.6, 9.00)
   {\textbf{Preserve pairwise influence scores (per layer)}\\[3pt]
    $\displaystyle \min_{\widehat W}
       \mathbb E_{z,z'}\!\left[\left|I_{\widehat W}(z,z')-I_W(z,z')\right|\right]$\\[3pt]
    \scriptsize $I_W(z,z')=\langle\nabla_W\ell(z),\nabla_W\ell(z')\rangle_F$};

\node[stepR, minimum width=6.8cm] (R2) at (12.6, 6.65)
   {\textbf{Influence-distortion bound}\\[3pt]
    $\displaystyle \mathbb E_{z,z'}|I_{\widehat W}-I_W|
       \;\le\; C_\kappa\sqrt{\mathbb E_z[(g^\top\Delta W^\top x)^2]}$\\[3pt]};

\node[stepR, minimum width=6.8cm] (R3) at (12.6, 4.10)
   {\textbf{Minimize the surrogate}\\[3pt]
    $\displaystyle \min_{\widehat W}
       \mathbb E_z[(g^\top\Delta W^\top x)^2]$\\[3pt]
    \scriptsize $x,g$ evaluated at the dense model};

\draw[arrR] (R1.south) -- (R2.north);
\draw[arrR] (R2.south) -- (R3.north);

\node[kmerge] (KMERGE) at (7.6, 1.40)
   {\textbf{Per-layer Kronecker moment approximation}\\[4pt]
    $\displaystyle
      \mathbb{E}\!\left[gg^{\!\top}\otimes xx^{\!\top}\right]
      \;\approx\;
      \mathbb{E}[gg^{\!\top}] \otimes \mathbb{E}[xx^{\!\top}]
      \;\propto\; \bfG \otimes \bfH$};

\draw[arrL, line width=1.1pt, rounded corners=8pt]
   (L3.south) -- (2.6, 2.40) -- (KMERGE.west);
\draw[arrR, line width=1.1pt, rounded corners=8pt]
   (R3.south) -- (12.6, 2.40) -- (KMERGE.east);

\node[share] (SHARED) at (7.6, -0.5)
   {{\footnotesize\textbf{Shared curvature-aware objective}}\\[3pt]
    $\displaystyle
      \bigl\| \bfX\,(\bfWhat - \bfW)\,\bfA^{\!\top}\bigr\|_F^2
      \;=\;
      \Tr{\bfH\,(\bfWhat - \bfW)\,\bfG\,(\bfWhat - \bfW)^{\!\top}}$};

\draw[arr, line width=1.0pt] (KMERGE.south) -- (SHARED.north);

\end{tikzpicture}}
\caption{\textbf{One objective, two derivations.}
LLM-Surgeon~\citep{ouderaa2024llmsurgeon} and IProX's IPSVD stage~\citep{chen2026influencepreservingiprox} lead to \cref{eq:kfac-loss}. The influence-distortion bound uses the assumptions of IProX, Proposition 4.1 and Appendix D; $C_\kappa$ is the resulting finite constant. Here, $\bfH=\bfX^\top\bfX$, and $\bfG=\bfA^\top\bfA$.}
\label{fig:derivation-convergence}
\end{figure}
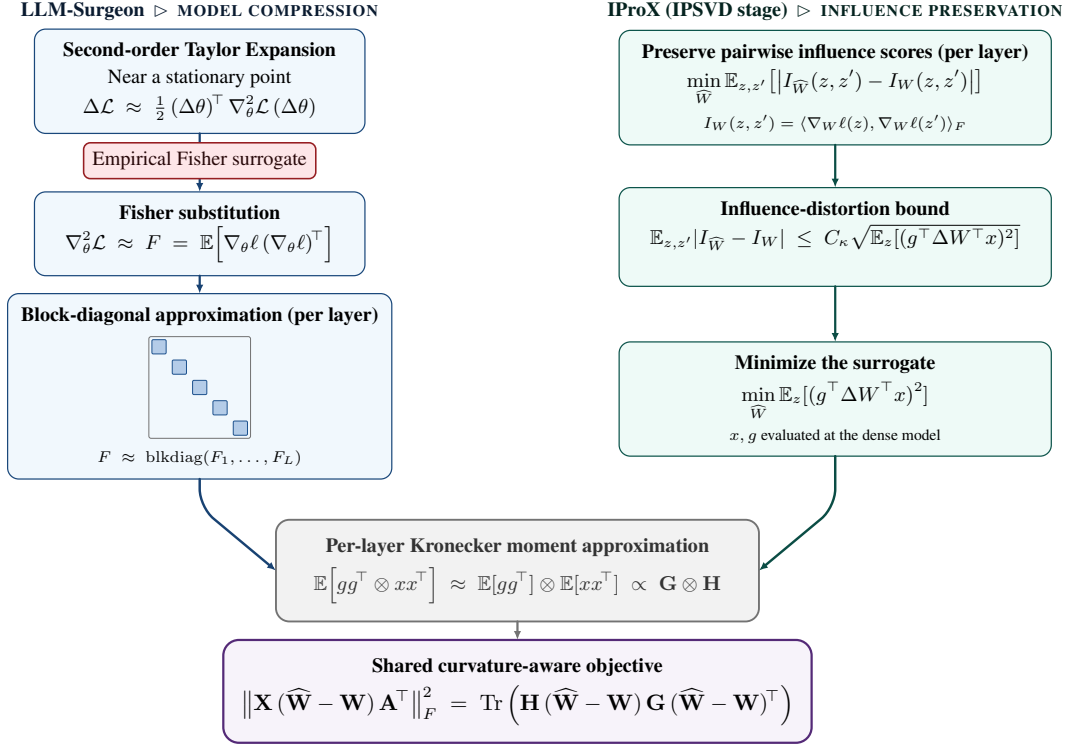

\clearpage

\section{Convergence of \ourmethod}\label{app:cactus-convergence}

\begin{proof}[{Proof of \cref{thm:cactus-convergence}}]
We follow 3BASiL~\citep[Appendix A]{makni20263basil}, identifying the variables used there $\widehat{\mathbf W},\mathbf D^{(t)},\mathbf V^{(t)}$ with our $\bfW,\bfDelta_t,\bfY_t=\rho_t\bfV_t$, respectively. The rescaled dual update gives $\bfY_{t+1}=\bfY_t+\rho_t(\bfS_{t+1}-\bfDelta_{t+1})$, exactly as in 3BASiL. Only the following two bounds change.

\paragraph{Low-rank bound (Lemma A.1).}
Replace 3BASiL's selection of the constant $C_F$ in Eq.~(9) as follows
\[
\max\big\{1,\|\bfW\|_F,\,2\|\bfH\|_2\|\bfG\|_2,\,2\|\bfH\bfW\bfG\|_F,
\|\bfH^{-1/2}\|_2\|\bfG^{-1/2}\|_2,
\|\bfH^{1/2}\|_2\|\bfG^{1/2}\|_2\big\}\le C_F.
\]
For an accepted candidate, \cref{eq:l-update-eigen} and $\|\mathrm{SVD}_r(\bfM)\|_F\le\|\bfM\|_F$ give
\[
\begin{aligned}
\|\bfL_t\|_F
&\le\|\bfH^{-1/2}\|_2\|\bfG^{-1/2}\|_2
      \|\bfH^{1/2}(\bfW-\bfS_t)\bfG^{1/2}\|_F\\
&\le C_F^2(\|\bfW\|_F+\|\bfS_t\|_F)
\le C_F^3(1+a_t),
\end{aligned}
\]
where, as in 3BASiL, we define
\[
a_t=\|\bfDelta_t\|_F+\frac{\|\bfY_t\|_F+\|\bfY_{t-1}\|_F}{\rho_{t-1}},
\qquad b_t=\max\left\{a_t,\frac{\|\bfL_t\|_F}{C_F^3}-1\right\}.
\]
The last inequality follows from $\bfS_t=\bfDelta_t+(\bfY_t-\bfY_{t-1})/\rho_{t-1}$. If the candidate is rejected, $\bfL_t=\bfL_{t-1}$, as in Lemma A.1.

\paragraph{Smooth-step bound (Lemma A.2).}
Set $\bfR_t=\bfS_{t+1}-\bfDelta_t+\bfY_t/\rho_t$. The calculation preceding Eq.~(23) is replaced by the Sylvester equation from \cref{eq:s-update}:
\[
2\bfH\bfR_t\bfG+\rho_t\bfR_t
=2\bfH(\bfW-\bfL_t-\bfDelta_t+\bfY_t/\rho_t)\bfG.
\]
In the two orthonormal eigenbases, each entry is divided by $2\lambda_{H,i}\lambda_{G,j}+\rho_t\ge\rho_t$. Thus Eqs.~(23) and (27) become
\[
\begin{aligned}
\|\bfR_t\|_F
&\le\frac{2}{\rho_t}
 \|\bfH(\bfW-\bfL_t-\bfDelta_t+\bfY_t/\rho_t)\bfG\|_F\\
&\le\frac{C_F(1+\|\bfDelta_t\|_F+\|\bfY_t\|_F/\rho_t)+C_F\|\bfL_t\|_F}{\rho_t}
\le\frac{(C_F+C_F^4)(1+b_t)}{\rho_t}.
\end{aligned}
\]
The projection inequality in Eq.~(24) then gives the bounds of Lemma A.2:
\[
\|\bfY_{t+1}\|_F\le(C_F+C_F^4)(1+b_t),
\qquad
\|\bfDelta_{t+1}-\bfDelta_t\|_F\le\frac{2(C_F+C_F^4)(1+b_t)}{\rho_t}.
\]

The proof of Lemma A.3 now gives the same recurrence:
\[
b_{t+1}\le\max\{a_{t+1},b_t\}
\le\left(1+\frac{3(C_F+C_F^4)}{\rho_{t-1}}\right)b_t
   +\frac{3(C_F+C_F^4)}{\rho_{t-1}}.
\]
The remainder of the 3BASiL argument is unchanged: \cref{eq:cactus-lowrank-gate} gives the same movement bound, and closedness of the constraint sets gives feasibility of the limits.
\end{proof}

\clearpage

\section{Additional Experimental Protocol Details}\label{sect:appendix-exp-details}
\paragraph{\ourmethod implementation.}
Before eigendecomposition, we apply separate ridge damping to the activation and gradient factors:
\begin{align*}
  \bfH &\leftarrow \bfH + \lambda_H\,\mathrm{mean}(\mathrm{diag}(\bfH))\bfI,\\
  \bfG &\leftarrow \bfG + \lambda_G\,\mathrm{mean}(\mathrm{diag}(\bfG))\bfI.
\end{align*}
Here $\lambda_H$ and $\lambda_G$ are damping parameters. The algorithm minimizes the two-sided objective in \cref{eq:kfac-loss} using these damped factors. We present the ppl results using different ridge parameters in \cref{tab:curvature-damping}. Compression seems to be robust to the ridge values used. For large layers/model architectures, $0.005$ can cause decomposition issues (poor conditioning). Except in the ridge-damping ablation, we use $\lambda_H=0.01$ (same as SparseGPT) and $\lambda_G=0.1$.
\begin{table}[H]
\centering
\small
\resizebox{0.4\textwidth}{!}{%
\begin{tabular}{r rrrr}
\toprule
& \multicolumn{4}{c}{$\lambda_G$} \\
\cmidrule(lr){2-5}
$\lambda_H$ & $0.005$ & $0.01$ & $0.1$ & $0.25$ \\
\midrule
\multicolumn{5}{c}{C4 perplexity ($\downarrow$)} \\
\midrule
$0.005$ & 19.729 & 19.647 & 19.448 & 19.449 \\
$0.01$ & 19.719 & 19.596 & 19.432 & \textbf{19.392} \\
$0.1$ & 19.626 & 19.601 & 19.536 & 19.395 \\
$0.25$ & 19.681 & 19.780 & 19.548 & 19.509 \\
\midrule
\multicolumn{5}{c}{WT-2 perplexity ($\downarrow$)} \\
\midrule
$0.005$ & 14.351 & 14.447 & 14.308 & 14.186 \\
$0.01$ & 14.346 & 14.302 & 14.168 & 14.070 \\
$0.1$ & 14.207 & 14.150 & 14.045 & 14.023 \\
$0.25$ & 14.218 & 14.195 & 14.097 & \textbf{13.972} \\
\bottomrule
\end{tabular}%
}
\caption{Ridge-damping ablation for $\bfH$ and $\bfG$ using next-token cross-entropy on Llama-3.2-1B at $50\%$ unstructured sparsity. Bold denotes the lowest perplexity within each dataset's grid.}
\label{tab:curvature-damping}
\end{table}


By default, our ADMM-based algorithm \ourmethod-slr uses $C_L=10$ for the low-rank movement condition in \cref{alg:kfac-admm}.

\begin{table}[H]
\centering
\small
\resizebox{0.6\textwidth}{!}{%
\begin{tabular}{r r r r r r}
\toprule
$C_L$ & Accept (\%) & Reject (\%) & Avg. iterations & C4 $\downarrow$ & WT-2 $\downarrow$ \\
\midrule
$\infty$ & 100.00 & 0.00 & 430 & 18.256 & 13.061 \\
$10$ (default) & 100.00 & 0.00 & 431 & 18.307 & 12.952 \\
$3$ & 100.00 & 0.00 & 429 & 18.290 & 13.021 \\
$1$ & 100.00 & 0.00 & 431 & 18.337 & 13.023 \\
$0.9$ & 99.99 & 0.01 & 430 & 18.293 & 12.909 \\
$0.8$ & 99.62 & 0.38 & 430 & 18.276 & 13.028 \\
$0.7$ & 97.47 & 2.53 & 429 & 18.320 & 13.014 \\
$0.6$ & 61.58 & 38.42 & 417 & 19.428 & 14.122 \\
$0.5$ & 22.03 & 77.97 & 407 & 24.511 & 18.527 \\
$0.4$ & 1.16 & 98.84 & 403 & 35.538 & 26.500 \\
$0.3$ & 0.00 & 100.00 & 403 & 36.723 & 27.272 \\
$0.1$ & 0.00 & 100.00 & 403 & 36.590 & 27.079 \\
\bottomrule
\end{tabular}}
\caption{Low-rank movement-threshold ablation: \ourmethod-\slr on Llama-3.2-1B at $50\%$ total compression with rank-budget allocation $\eta=0.3$. Rejection rates and iteration counts are averaged over all $112$ compressed sublayers.}
\label{tab:low-rank-movement}
\end{table}

We use an initial penalty $\rho_0=0.1$ and update the ADMM penalty using the support-change rule from ALPS/3BASiL. Let $s_t = |\operatorname{Supp}(\bfDelta^{(t)}) \triangle \operatorname{Supp}(\bfDelta^{(t-m)})|$ be the symmetric difference in sparse support at the update period $m$, and let $k$ be the number of retained nonzeros. The finite-run support-change schedules are summarized in \cref{tab:cactus-rho-schedule}. For the non-decreasing schedule analyzed in \cref{thm:cactus-convergence}, a multiplier of at least $\eta_{\min}=1.02$ at every fixed-period update gives $\rho_t\ge\rho_0\eta_{\min}^{\lfloor t/m\rfloor}$, so $\sum_{t=0}^{\infty}1/\rho_t\le m\eta_{\min}/[\rho_0(\eta_{\min}-1)]<\infty$.
\begin{table}[H]
\centering
\small
\resizebox{0.7\textwidth}{!}{%
\begin{tabular}{lccc}
\toprule
\textbf{Setting} & \textbf{Max iter.} & \textbf{Support-change condition} & \textbf{$\rho$ multipliers} \\
\midrule
Pure pruning & $100$ & $s_t/k>0.1,\;s_t/k>0.005,\;s_t>0.5$ & $1.3,\;1.2,\;1.1$ \\
\slr & $2000$ & $s_t/k>0.1,\;s_t/k>0.005,\;s_t>0.5$ & $1.1,\;1.05,\;1.02$ \\
\bottomrule
\end{tabular}
}
\caption{Support-change penalty multipliers for compression. At the update period $m$, the multiplier is selected according to which support-change condition holds: the first matching threshold sets the multiplier $\rho^{(t+1)}/\rho^{(t)}$. The \slr schedule is intentionally less aggressive as \slr decompositions seem to be more sensitive in our experiments.}
\label{tab:cactus-rho-schedule}
\end{table}

\paragraph{Baselines.}
All one-shot baselines are run in the same layerwise compression framework and with the same calibration samples, except when a method explicitly includes a post-compression stage. \texttt{SparseGPT}, \texttt{ALPS}, \texttt{SVD-LLM(W)} (simple activation-aware low-rank fitting without sequential factor refinement or post-compression fine-tuning), \texttt{OATS}, and \texttt{3BASiL} use activation-only objectives using $\bfH$, whereas \texttt{IPSVD}, \texttt{IProX}, \texttt{LLM-Surgeon}, and \ourmethod use $\bfG$ as well. For \slr baselines at target compression ratio $c$, we match total parameter budget using the same layerwise rank allocation $r_\ell=\lfloor \eta(1-c)m_\ell n_\ell/(m_\ell+n_\ell)\rfloor$ with $\eta=0.3$; for named structured settings such as $(2{:}4{+}r256)$, the sparsity pattern and rank are fixed exactly as specified.

For LLM-Surgeon, we use row-wise correlated weight updates (\texttt{max\_correlate=-1}), jointly accounting for the pruned weights within each row when computing weight compensation, for both unstructured and 2:4 pruning. We also employ the five-shot setting described in \cref{sect:exp-setup}, since the one-shot unstructured pruning on Llama3.2-1B produced weak results. For Qwen3-4B model under 2:4 sparsity on MLP layers, one-shot results are competitive, so we report both one-shot and five-shot results.

\paragraph{Compression Protocol.}
We estimate activation and gradient covariance factors from $128$ C4 calibration sequences at sequence length $2048$ (see effect of number of calibration data points at \cref{tab:calibration-size}). This is a standard calibration set introduced by SparseGPT~\citep{frantar2023sparsegpt} and adopted by most layerwise reconstruction algorithms like ALPS~\citep{meng2024alps}. We evaluate language model utility with C4 and WikiText-2 perplexity at sequence length $2048$. Influence-preservation metrics use projected per-example gradients (via TRAK's \texttt{fast\_jl}) with the same seeded projection. Thus, TracIn entries are computed as $\langle \Pi g_i,\Pi g_j\rangle$ under a fixed JL sketch $\Pi$, which preserves inner products and ensures dense and compressed gradients are compared in the same projected space.
The per-sample gradients are sampled from $n_{\text{tr}}=1000$ C4 training sequences (different from calibration samples) and $n_{\text{te}}=100$ WikiText-2 test sequences, both at sequence length $512$.

\paragraph{Data Selection Protocol.}
We follow the LESS~\citep{xia2024less} data-selection procedure on its released FLAN-v2, CoT, Dolly, and OpenAssistant mixture (270,679 candidates), comparing dense and compressed scoring models. Scorers share a $5\%$ warm-up sample  and train LoRA adapters on attention $q/k/v/o$ projections for four epochs. Warm-up and target training use learning rate $2\times10^{-5}$, effective batch size $128$, sequence length $2048$, cosine decay, and $3\%$ learning-rate warm-up. Similar to the LESS framework, we apply a shared $8192$-dimensional TRAK projection to Adam-aware candidate updates and ordinary query gradients in the LoRA parameter space.

In \cref{tab:less-bbh-selection}, Spearman compares aggregate scores on common valid candidates; overlap measures intersection with dense LESS divided by $13{,}533$ (this corresponds to $5\%$ of the entire dataset size). \Cref{fig:less-bbh-score-curve} plots 3B scores in dense-ranked order. Training on the authors' released 7B LESS subset yields $41.30\%$ BBH-EM.

\paragraph{LoGRA Integration Protocol.}
The setup in LoGRA Integration subsection reuses the standard one-shot calibration setup ($128$ C4 sequences at sequence length $2048$, $\bfH$ and $\bfG$ collected with the same damping values as the one-shot compression subsection) and compresses Qwen3-4B with each method restricted to MLP linears (\texttt{gate\_proj}, \texttt{up\_proj}, \texttt{down\_proj}). Hence, the attention $q/k/v/o$ remain dense. Sparsity is set to $2{:}4$; the \slr experiment additionally adds rank-$64$ low-rank corrections using \ourmethod-SLR on the same MLP layers. For \texttt{LLM-Surgeon}, we report both the 1-shot and 5-shot variants. The LoGRA influence-evaluation pipeline is identical for every row. We attach rank-$64$ LoGRA adapters to every tracked MLP linear (\texttt{gate\_proj}, \texttt{up\_proj}, and \texttt{down\_proj}) of the deployed compressed model with $\bfP_o\!\in\!\R^{d_{\text{out}}\times r}$ and $\bfP_i\!\in\!\R^{r\times d_{\text{in}}}$ drawn from a single shared seed (so all methods share the \emph{same} random subspaces and IF entries are pairwise comparable).

\textbf{Concat-LoGRA construction.} For \ourmethod-SLR, we widen each tracked MLP-linear adapter to rank $2r$ by concatenating the SVD basis of that linear layer's low-rank correction denoted by $\bfL^\star_\ell{=}\bfU^\star_\ell\bm\Sigma^\star_\ell(\bfV_\ell^\star)^\top$:
$\bfA_{\text{cat},\ell}{=}[\,\bfP_o\,|\,\bfV^\star_\ell\,]$,
$\bfB_{\text{cat},\ell}{=}[\,\bfP_i^\top\,|\,\bfU_\ell^\star\,]^\top$.
Forward and backward use the full rank-$2r$ adapter with $\bfE_{\text{cat},\ell}{=}\mathrm{blkdiag}(\mathbf 0,\bm\Sigma^\star_\ell)$; $\bfA_{\text{cat},\ell}\bfE_{\text{cat},\ell}\bfB_{\text{cat},\ell}=(\bfL^\star_\ell)^\top$, so this is forward-equivalent to the deployed weight $\bfS_\ell{+}\bfL_\ell$ since the random block contributes zero to the forward. We capture the $2r\times 2r$ per-sample gradient and \emph{slice the top-left $r\times r$ block} as the IF statistic, corresponding to a standard rank-$r$ LoGRA gradient (of size $r \times r$), so all downstream IF analysis is identical.

\clearpage

\section{Additional Experimental Results}\label{app:if-ci}
This appendix reports the four influence-preservation metrics introduced in \cref{sect:exp-setup} ($\cos_{\text{tr}}, \cos_{\text{te}}, \rho_{\text{qry}}, R_{\text{qry@5\%}}$) for all nine compression methods at three compression levels ($30\%, 50\%, 70\%$). Each cell is annotated as mean $\pm\,\widehat{\mathrm{SEM}}$: the sample standard deviation divided by the square root of the number of evaluation examples. \textbf{bold} denotes the best within a constraint family; \underline{\textbf{underline}} denotes the best overall at a given compression rate. Emphasis applies only to the mean, with ties at the reported precision marked equally. Emphasis applies only to the mean, with ties at the reported precision marked equally. Runtime entries marked with $^{\dagger}$ use two NVIDIA A100 GPUs; other one-shot compression runtimes use a single L40 GPU unless stated otherwise.

\begin{table}[H]
\centering
\scriptsize
\setlength{\tabcolsep}{1.5pt}

\newcolumntype{C}[1]{>{\centering\arraybackslash}p{#1}}
\resizebox{0.90\textwidth}{!}{%
\begin{tabular}{l|C{3.0em}C{3.0em}|C{7.1em}C{7.1em}C{7.1em}C{7.1em}|C{4.2em}}
\toprule
\textbf{Method}
& \cellcolor{gray!10}\textbf{C4 PPL} ($\downarrow$)
& \cellcolor{gray!10}\textbf{WT-2 PPL} ($\downarrow$)
& \multicolumn{4}{c|}{\textbf{TracIn} ($\uparrow$)}
& \multirow{2}{*}{\shortstack{\textbf{Compress}\\\textbf{Runtime}}} \\
\cmidrule(lr){4-7}
& \cellcolor{gray!10}
& \cellcolor{gray!10}
& $\cos_{\text{tr}}$ & $\cos_{\text{te}}$ & $\rho_{\text{qry}}$ & $R_{\text{qry@5\%}}$ & \\
\midrule
\rowcolor{gray!10} \textbf{Dense} & 14.01 & 9.75 & 1.000 & 1.000 & 1.000 & 1.000 & --- \\
\midrule[\heavyrulewidth]
\multicolumn{8}{c}{\cellcolor{blue!8}\textbf{30\% compression}} \\
\midrule
\multicolumn{8}{l}{\cellcolor{blue!4}\emph{Pure pruning}} \\
SparseGPT          & 15.94 & 10.92 & 0.737\,$\pm$\,0.005 & 0.730\,$\pm$\,0.018 & 0.524\,$\pm$\,0.014 & 0.293\,$\pm$\,0.007 & 10m 24s \\
ALPS               & 15.53 & 10.71 & 0.773\,$\pm$\,0.005 & 0.764\,$\pm$\,0.017 & 0.580\,$\pm$\,0.014 & 0.347\,$\pm$\,0.007 & 14m 48s \\
LLM-Surgeon(5sh)   & 15.18 & \textbf{10.48} & 0.817\,$\pm$\,0.005 & 0.802\,$\pm$\,0.017 & 0.644\,$\pm$\,0.015 & 0.391\,$\pm$\,0.008 & 4h 57m\\
\rowcolor{darkgreen!10}
\texttt{CACTUS-prune}
& \textbf{15.03} & 10.49 & \textbf{0.840}\,$\pm$\,0.005 & \textbf{0.824}\,$\pm$\,0.017 & \textbf{0.685}\,$\pm$\,0.015 & \textbf{0.428}\,$\pm$\,0.009 & 15m 22s \\
\midrule
\multicolumn{8}{l}{\cellcolor{blue!4}\emph{Pure low-rank}} \\
SVD-LLM(W)         & 135.87 & 184.95 & 0.029\,$\pm$\,0.001 & 0.029\,$\pm$\,0.003 & 0.013\,$\pm$\,0.004 & 0.047\,$\pm$\,0.003 & 18m 15s \\
IPSVD              & 72.01 & 91.87 & 0.098\,$\pm$\,0.002 & 0.096\,$\pm$\,0.004 & 0.028\,$\pm$\,0.005 & 0.050\,$\pm$\,0.003 & 18m 54s \\
IProX         & \textbf{53.05} & \textbf{81.46} & \textbf{0.173}\,$\pm$\,0.002 & \textbf{0.171}\,$\pm$\,0.005 & \textbf{0.067}\,$\pm$\,0.004 & \textbf{0.094}\,$\pm$\,0.004 & 36m 41s \\
\midrule
\multicolumn{8}{l}{\cellcolor{blue!4} \slr} \\
OATS               & 15.77 & 10.78 & 0.754\,$\pm$\,0.005 & 0.740\,$\pm$\,0.017 & 0.542\,$\pm$\,0.014 & 0.315\,$\pm$\,0.007 & 36m 46s \\
3BASiL             & 14.92 & 10.27 & 0.845\,$\pm$\,0.004 & 0.835\,$\pm$\,0.015 & 0.678\,$\pm$\,0.015 & 0.428\,$\pm$\,0.008 & 50m 07s \\
\rowcolor{darkgreen!10}
\texttt{CACTUS-slr}
& \underline{\textbf{14.64}} & \underline{\textbf{10.12}} & \underline{\textbf{0.892}}\,$\pm$\,0.004 & \underline{\textbf{0.875}}\,$\pm$\,0.015 & \underline{\textbf{0.766}}\,$\pm$\,0.014 & \underline{\textbf{0.525}}\,$\pm$\,0.010 & 59m 31s \\
\midrule[\heavyrulewidth]
\multicolumn{8}{c}{\cellcolor{blue!8}\textbf{50\% compression}} \\
\midrule
\multicolumn{8}{l}{\cellcolor{blue!4}\emph{Pure pruning}} \\
SparseGPT          & 26.54 & 19.61 & 0.266\,$\pm$\,0.003 & 0.267\,$\pm$\,0.011 & 0.036\,$\pm$\,0.005 & 0.083\,$\pm$\,0.005 & 10m 54s \\
ALPS               & 22.20 & 16.44 & 0.395\,$\pm$\,0.004 & 0.394\,$\pm$\,0.012 & 0.130\,$\pm$\,0.007 & 0.117\,$\pm$\,0.005 & 14m 51s \\
LLM-Surgeon(5sh)   & 21.39 & 15.04 & 0.451\,$\pm$\,0.004 & 0.445\,$\pm$\,0.014 & 0.184\,$\pm$\,0.008 & 0.130\,$\pm$\,0.006 & 5h 18m\\
\rowcolor{darkgreen!10}
\texttt{CACTUS-prune}
& \textbf{19.43} & \textbf{14.17} & \textbf{0.524}\,$\pm$\,0.005 & \textbf{0.512}\,$\pm$\,0.016 & \textbf{0.252}\,$\pm$\,0.010 & \textbf{0.152}\,$\pm$\,0.006 & 16m 26s \\
\midrule
\multicolumn{8}{l}{\cellcolor{blue!4}\emph{Pure low-rank}} \\
SVD-LLM(W)         & 403.18 & 632.93 & 0.016\,$\pm$\,0.001 & 0.015\,$\pm$\,0.002 & 0.018\,$\pm$\,0.004 & 0.072\,$\pm$\,0.004 & 18m 03s \\
IPSVD              & 198.39 & 322.80 & 0.042\,$\pm$\,0.001 & 0.040\,$\pm$\,0.003 & 0.019\,$\pm$\,0.004 & 0.060\,$\pm$\,0.004 & 18m 52s \\
IProX         & \textbf{68.82} & \textbf{102.61} & \textbf{0.151}\,$\pm$\,0.002 & \textbf{0.145}\,$\pm$\,0.005 & \textbf{0.075}\,$\pm$\,0.004 & \textbf{0.090}\,$\pm$\,0.005 & 34m 28s \\
\midrule
\multicolumn{8}{l}{\cellcolor{blue!4} \slr} \\
OATS               & 29.29 & 20.47 & 0.209\,$\pm$\,0.003 & 0.214\,$\pm$\,0.007 & 0.036\,$\pm$\,0.004 & 0.070\,$\pm$\,0.005 & 35m 44s \\
3BASiL             & 19.98 & 14.46 & 0.475\,$\pm$\,0.004 & 0.458\,$\pm$\,0.016 & 0.227\,$\pm$\,0.010 & 0.154\,$\pm$\,0.005 & 43m 06s \\
\rowcolor{darkgreen!10}
\texttt{CACTUS-slr}
& \underline{\textbf{18.27}} & \underline{\textbf{13.02}} & \underline{\textbf{0.585}}\,$\pm$\,0.005 & \underline{\textbf{0.569}}\,$\pm$\,0.017 & \underline{\textbf{0.327}}\,$\pm$\,0.011 & \underline{\textbf{0.199}}\,$\pm$\,0.006 & 46m 51s \\
\midrule[\heavyrulewidth]
\multicolumn{8}{c}{\cellcolor{blue!8}\textbf{70\% compression}} \\
\midrule
\multicolumn{8}{l}{\cellcolor{blue!4}\emph{Pure pruning}} \\
SparseGPT          & 156.90 & 139.13 & 0.030\,$\pm$\,0.001 & 0.030\,$\pm$\,0.003 & 0.021\,$\pm$\,0.004 & 0.049\,$\pm$\,0.003 & 10m 25s \\
ALPS               & 68.21 &  63.47 & 0.065\,$\pm$\,0.001 & 0.071\,$\pm$\,0.003 & 0.021\,$\pm$\,0.004 & 0.049\,$\pm$\,0.003 & 10m 45s \\
LLM-Surgeon(5sh)   & 64.56 &  52.02 & 0.104\,$\pm$\,0.002 & 0.107\,$\pm$\,0.005 & 0.015\,$\pm$\,0.006 & 0.055\,$\pm$\,0.004 & 5h 33m$^{\dagger}$ \\
\rowcolor{darkgreen!10}
\texttt{CACTUS-prune}
& \underline{\textbf{52.19}} & \underline{\textbf{41.08}} & \textbf{0.122}\,$\pm$\,0.002 & \underline{\textbf{0.124}}\,$\pm$\,0.004 & \textbf{0.043}\,$\pm$\,0.005 & \textbf{0.068}\,$\pm$\,0.004 & 17m 09s \\
\midrule
\multicolumn{8}{l}{\cellcolor{blue!4}\emph{Pure low-rank}} \\
SVD-LLM(W)         & 1081 & 2106 & 0.010\,$\pm$\,0.001 & 0.008\,$\pm$\,0.002 & 0.016\,$\pm$\,0.004 & 0.075\,$\pm$\,0.004 & 18m 59s \\
IPSVD              & 1095 & 1896 & 0.006\,$\pm$\,0.001 & 0.006\,$\pm$\,0.002 & 0.009\,$\pm$\,0.004 & 0.051\,$\pm$\,0.004 & 18m 13s \\
IProX              & \textbf{156.30} & \textbf{293.62} & \textbf{0.075}\,$\pm$\,0.001 & \textbf{0.073}\,$\pm$\,0.003 & \underline{\textbf{0.051}}\,$\pm$\,0.004 & \underline{\textbf{0.104}}\,$\pm$\,0.005 & 28m 18s \\
\midrule
\multicolumn{8}{l}{\cellcolor{blue!4} \slr} \\
OATS               & 389.61 & 355.95 & 0.019\,$\pm$\,0.001 & 0.023\,$\pm$\,0.002 & 0.028\,$\pm$\,0.004 & 0.062\,$\pm$\,0.004 & 35m 20s \\
3BASiL             & 79.98 &  68.06 & 0.045\,$\pm$\,0.001 & 0.047\,$\pm$\,0.003 & 0.021\,$\pm$\,0.004 & 0.048\,$\pm$\,0.003 & 38m 01s \\
\rowcolor{darkgreen!10}
\texttt{CACTUS-slr}
& \textbf{54.58} & \textbf{41.74} & \underline{\textbf{0.126}}\,$\pm$\,0.002 & \underline{\textbf{0.124}}\,$\pm$\,0.005 & \textbf{0.045}\,$\pm$\,0.006 & \textbf{0.070}\,$\pm$\,0.004 & 37m 12s \\
\bottomrule
\end{tabular}%
}
\vspace{3pt}
\caption{\textbf{Llama-3.2-1B.} Influence-preservation results in the \textbf{TracIn} (unweighted) representation.}
\label{tab:main-sweep-appendix-tracin}
\end{table}

\clearpage

\begin{table}[H]
\centering
\scriptsize
\setlength{\tabcolsep}{1.5pt}

\newcolumntype{C}[1]{>{\centering\arraybackslash}p{#1}}
\resizebox{0.90\textwidth}{!}{%
\begin{tabular}{l|C{3.0em}C{3.0em}|C{7.1em}C{7.1em}C{7.1em}C{7.1em}|C{4.2em}}
\toprule
\textbf{Method}
& \cellcolor{gray!10}\textbf{C4 PPL} ($\downarrow$)
& \cellcolor{gray!10}\textbf{WT-2 PPL} ($\downarrow$)
& \multicolumn{4}{c|}{\textbf{KFAC} ($\uparrow$)}
& \multirow{2}{*}{\shortstack{\textbf{Compress}\\\textbf{Runtime}}} \\
\cmidrule(lr){4-7}
& \cellcolor{gray!10}
& \cellcolor{gray!10}
& $\cos_{\text{tr}}$ & $\cos_{\text{te}}$ & $\rho_{\text{qry}}$ & $R_{\text{qry@5\%}}$ & \\
\midrule
\rowcolor{gray!10} \textbf{Dense} & 14.01 & 9.75 & 1.000 & 1.000 & 1.000 & 1.000 & --- \\
\midrule[\heavyrulewidth]
\multicolumn{8}{c}{\cellcolor{blue!8}\textbf{30\% compression}} \\
\midrule
\multicolumn{8}{l}{\cellcolor{blue!4}\emph{Pure pruning}} \\
SparseGPT          & 15.94 & 10.92 & 0.689\,$\pm$\,0.005 & 0.690\,$\pm$\,0.016 & 0.429\,$\pm$\,0.011 & 0.277\,$\pm$\,0.007 & 10m 24s \\
ALPS               & 15.53 & 10.71 & 0.729\,$\pm$\,0.005 & 0.724\,$\pm$\,0.016 & 0.482\,$\pm$\,0.011 & 0.297\,$\pm$\,0.008 & 14m 48s \\
LLM-Surgeon(5sh)   & 15.18 & \textbf{10.48} & 0.772\,$\pm$\,0.005 & 0.762\,$\pm$\,0.016 & 0.540\,$\pm$\,0.012 & 0.319\,$\pm$\,0.008 & 4h 57m \\
\rowcolor{darkgreen!10}
\texttt{CACTUS-prune}
& \textbf{15.03} & 10.49 & \textbf{0.794}\,$\pm$\,0.005 & \textbf{0.784}\,$\pm$\,0.017 & \textbf{0.570}\,$\pm$\,0.013 & \textbf{0.337}\,$\pm$\,0.008 & 15m 22s \\
\midrule
\multicolumn{8}{l}{\cellcolor{blue!4}\emph{Pure low-rank}} \\
SVD-LLM(W)         & 135.87 & 184.95 & 0.012\,$\pm$\,0.001 & 0.010\,$\pm$\,0.001 & 0.004\,$\pm$\,0.003 & 0.046\,$\pm$\,0.003 & 18m 15s \\
IPSVD              & 72.01 & 91.87 & 0.050\,$\pm$\,0.001 & 0.051\,$\pm$\,0.003 & 0.005\,$\pm$\,0.004 & 0.051\,$\pm$\,0.004 & 18m 54s \\
IProX         & \textbf{53.05} & \textbf{81.46} & \textbf{0.097}\,$\pm$\,0.001 & \textbf{0.095}\,$\pm$\,0.004 & \textbf{0.015}\,$\pm$\,0.003 & \textbf{0.081}\,$\pm$\,0.004 & 36m 41s \\
\midrule
\multicolumn{8}{l}{\cellcolor{blue!4} \slr} \\
OATS               & 15.77 & 10.78 & 0.713\,$\pm$\,0.004 & 0.703\,$\pm$\,0.016 & 0.458\,$\pm$\,0.011 & 0.275\,$\pm$\,0.006 & 36m 46s \\
3BASiL             & 14.92 & 10.27 & 0.814\,$\pm$\,0.004 & 0.810\,$\pm$\,0.015 & 0.609\,$\pm$\,0.012 & 0.388\,$\pm$\,0.009 & 50m 07s \\
\rowcolor{darkgreen!10}
\texttt{CACTUS-slr}
& \underline{\textbf{14.64}} & \underline{\textbf{10.12}} & \underline{\textbf{0.862}}\,$\pm$\,0.004 & \underline{\textbf{0.849}}\,$\pm$\,0.014 & \underline{\textbf{0.687}}\,$\pm$\,0.012 & \underline{\textbf{0.447}}\,$\pm$\,0.010 & 59m 31s \\
\midrule[\heavyrulewidth]
\multicolumn{8}{c}{\cellcolor{blue!8}\textbf{50\% compression}} \\
\midrule
\multicolumn{8}{l}{\cellcolor{blue!4}\emph{Pure pruning}} \\
SparseGPT          & 26.54 & 19.61 & 0.235\,$\pm$\,0.003 & 0.235\,$\pm$\,0.009 & 0.028\,$\pm$\,0.004 & 0.083\,$\pm$\,0.004 & 10m 54s \\
ALPS               & 22.20 & 16.44 & 0.336\,$\pm$\,0.003 & 0.333\,$\pm$\,0.010 & 0.090\,$\pm$\,0.004 & 0.120\,$\pm$\,0.005 & 14m 51s \\
LLM-Surgeon(5sh)   & 21.39 & 15.04 & 0.375\,$\pm$\,0.004 & 0.371\,$\pm$\,0.012 & 0.118\,$\pm$\,0.005 & 0.126\,$\pm$\,0.005 & 5h 18m \\
\rowcolor{darkgreen!10}
\texttt{CACTUS-prune}
& \textbf{19.43} & \textbf{14.17} & \textbf{0.444}\,$\pm$\,0.004 & \textbf{0.434}\,$\pm$\,0.014 & \textbf{0.171}\,$\pm$\,0.007 & \textbf{0.136}\,$\pm$\,0.005 & 16m 26s \\
\midrule
\multicolumn{8}{l}{\cellcolor{blue!4}\emph{Pure low-rank}} \\
SVD-LLM(W)         & 403.18 & 632.93 & 0.004\,$\pm$\,0.001 & 0.005\,$\pm$\,0.001 & 0.001\,$\pm$\,0.004 & 0.060\,$\pm$\,0.004 & 18m 03s \\
IPSVD              & 198.39 & 322.80 & 0.012\,$\pm$\,0.001 & 0.012\,$\pm$\,0.001 & 0.004\,$\pm$\,0.004 & 0.048\,$\pm$\,0.003 & 18m 52s \\
IProX         & \textbf{68.82} & \textbf{102.61} & \textbf{0.060}\,$\pm$\,0.001 & \textbf{0.060}\,$\pm$\,0.002 & \textbf{0.009}\,$\pm$\,0.003 & \textbf{0.078}\,$\pm$\,0.004 & 34m 28s \\
\midrule
\multicolumn{8}{l}{\cellcolor{blue!4} \slr} \\
OATS               & 29.29 & 20.47 & 0.179\,$\pm$\,0.002 & 0.185\,$\pm$\,0.006 & 0.017\,$\pm$\,0.004 & 0.057\,$\pm$\,0.004 & 35m 44s \\
3BASiL             & 19.98 & 14.46 & 0.409\,$\pm$\,0.004 & 0.396\,$\pm$\,0.014 & 0.140\,$\pm$\,0.006 & 0.137\,$\pm$\,0.005 & 43m 06s \\
\rowcolor{darkgreen!10}
\texttt{CACTUS-slr}
& \underline{\textbf{18.27}} & \underline{\textbf{13.02}} & \underline{\textbf{0.505}}\,$\pm$\,0.005 & \underline{\textbf{0.493}}\,$\pm$\,0.016 & \underline{\textbf{0.219}}\,$\pm$\,0.008 & \underline{\textbf{0.163}}\,$\pm$\,0.005 & 46m 51s \\
\midrule[\heavyrulewidth]
\multicolumn{8}{c}{\cellcolor{blue!8}\textbf{70\% compression}} \\
\midrule
\multicolumn{8}{l}{\cellcolor{blue!4}\emph{Pure pruning}} \\
SparseGPT          & 156.90 & 139.13 & 0.012\,$\pm$\,0.001 & 0.011\,$\pm$\,0.001 & 0.002\,$\pm$\,0.004 & 0.033\,$\pm$\,0.003 & 10m 25s \\
ALPS               & 68.21 &  63.47 & 0.037\,$\pm$\,0.001 & 0.040\,$\pm$\,0.002 & 0.004\,$\pm$\,0.003 & 0.041\,$\pm$\,0.003 & 10m 45s \\
LLM-Surgeon(5sh)   & 64.56 &  52.02 & 0.063\,$\pm$\,0.001 & 0.064\,$\pm$\,0.003 & 0.003\,$\pm$\,0.004 & 0.039\,$\pm$\,0.004 & 5h 33m \\
\rowcolor{darkgreen!10}
\texttt{CACTUS-prune}
& \underline{\textbf{52.19}} & \underline{\textbf{41.08}} & \underline{\textbf{0.071}}\,$\pm$\,0.001 & \underline{\textbf{0.072}}\,$\pm$\,0.003 & \textbf{0.009}\,$\pm$\,0.003 & \textbf{0.055}\,$\pm$\,0.004 & 17m 09s \\
\midrule
\multicolumn{8}{l}{\cellcolor{blue!4}\emph{Pure low-rank}} \\
SVD-LLM(W)         & 1081 & 2106 & 0.002\,$\pm$\,0.001 & 0.001\,$\pm$\,0.001 & 0.002\,$\pm$\,0.003 & 0.067\,$\pm$\,0.004 & 18m 59s \\
IPSVD              & 1095 & 1896 & 0.001\,$\pm$\,0.001 & 0.003\,$\pm$\,0.001 & 0.002\,$\pm$\,0.004 & 0.046\,$\pm$\,0.004 & 18m 13s \\
IProX              & \textbf{156.30} & \textbf{293.62} & \textbf{0.020}\,$\pm$\,0.001 & \textbf{0.019}\,$\pm$\,0.002 & \underline{\textbf{0.010}}\,$\pm$\,0.003 & \underline{\textbf{0.094}}\,$\pm$\,0.004 & 28m 18s \\
\midrule
\multicolumn{8}{l}{\cellcolor{blue!4} \slr} \\
OATS               & 389.61 & 355.95 & 0.006\,$\pm$\,0.001 & 0.006\,$\pm$\,0.002 & $-0.002$\,$\pm$\,0.003 & 0.044\,$\pm$\,0.003 & 35m 20s \\
3BASiL             & 79.98 &  68.06 & 0.024\,$\pm$\,0.001 & 0.026\,$\pm$\,0.002 & 0.002\,$\pm$\,0.003 & 0.037\,$\pm$\,0.003 & 38m 01s \\
\rowcolor{darkgreen!10}
\texttt{CACTUS-slr}
& \textbf{54.58} & \textbf{41.74} & \underline{\textbf{0.071}}\,$\pm$\,0.001 & \underline{\textbf{0.072}}\,$\pm$\,0.003 & \textbf{0.009}\,$\pm$\,0.003 & \textbf{0.058}\,$\pm$\,0.004 & 37m 12s \\
\bottomrule
\end{tabular}%
}
\vspace{3pt}
\caption{\textbf{Llama-3.2-1B.} Influence-preservation results in the \textbf{KFAC} representation.}
\label{tab:main-sweep-appendix-kfac}
\end{table}

\clearpage

\paragraph{Qwen3-4B compression results.} Tables~\ref{tab:main-sweep-appendix-tracin-q4b}--\ref{tab:main-sweep-appendix-kfac-q4b} report the same comparison of nine methods on Qwen3-4B at the three compression levels ($30\%, 50\%, 70\%$), under the identical evaluation protocol of Section~\ref{sect:exp-setup} ($n_{\text{tr}}=1000$ C4, $n_{\text{te}}=100$ WT-2, projection dim $d=8192$). Each cell is annotated with $\pm\,\widehat{\mathrm{SEM}}$ following the same conventions as Tables~\ref{tab:main-sweep-appendix-tracin}.

\begin{table}[H]
\centering
\scriptsize
\setlength{\tabcolsep}{1.5pt}

\newcolumntype{C}[1]{>{\centering\arraybackslash}p{#1}}
\resizebox{0.90\textwidth}{!}{%
\begin{tabular}{l|C{3.0em}C{3.0em}|C{7.1em}C{7.1em}C{7.1em}C{7.1em}|C{4.2em}}
\toprule
\textbf{Method}
& \cellcolor{gray!10}\textbf{C4 PPL} ($\downarrow$)
& \cellcolor{gray!10}\textbf{WT-2 PPL} ($\downarrow$)
& \multicolumn{4}{c|}{\textbf{TracIn} ($\uparrow$)}
& \multirow{2}{*}{\shortstack{\textbf{Compress}\\\textbf{Runtime}}} \\
\cmidrule(lr){4-7}
& \cellcolor{gray!10}
& \cellcolor{gray!10}
& $\cos_{\text{tr}}$ & $\cos_{\text{te}}$ & $\rho_{\text{qry}}$ & $R_{\text{qry@5\%}}$ & \\
\midrule
\rowcolor{gray!10} \textbf{Dense} & 19.85 & 13.66 & 1.000 & 1.000 & 1.000 & 1.000 & --- \\
\midrule[\heavyrulewidth]
\multicolumn{8}{c}{\cellcolor{blue!8}\textbf{30\% compression}} \\
\midrule
\multicolumn{8}{l}{\cellcolor{blue!4}\emph{Pure pruning}} \\
SparseGPT                & 21.05 & 14.53 & 0.852\,$\pm$\,0.004 & 0.862\,$\pm$\,0.008 & 0.959\,$\pm$\,0.002 & 0.746\,$\pm$\,0.006 & 24m 53s \\
ALPS                     & 20.44 & \textbf{13.84} & 0.892\,$\pm$\,0.003 & 0.903\,$\pm$\,0.006 & 0.977\,$\pm$\,0.001 & 0.820\,$\pm$\,0.004 & 43m 07s \\
LLM-Surgeon(5sh)         & 20.44 & 14.33 & 0.910\,$\pm$\,0.004 & 0.915\,$\pm$\,0.008 & 0.971\,$\pm$\,0.003 & 0.829\,$\pm$\,0.009 & 18h 41m$^{\dagger}$ \\
\rowcolor{darkgreen!10}
\texttt{CACTUS-prune}      & \textbf{20.21} & 13.87 & \textbf{0.920}\,$\pm$\,0.003 & \textbf{0.923}\,$\pm$\,0.005 & \textbf{0.984}\,$\pm$\,0.001 & \textbf{0.837}\,$\pm$\,0.003 & 45m 59s \\
\midrule
\multicolumn{8}{l}{\cellcolor{blue!4}\emph{Pure low-rank}} \\
SVD-LLM(W)               & 47.31 & 70.50 & 0.178\,$\pm$\,0.002 & 0.183\,$\pm$\,0.007 & 0.309\,$\pm$\,0.007 & 0.103\,$\pm$\,0.004 & 54m 43s \\
IPSVD                    & \textbf{37.79} & \textbf{60.98} & \textbf{0.378}\,$\pm$\,0.004 & \textbf{0.383}\,$\pm$\,0.013 & \textbf{0.545}\,$\pm$\,0.011 & \textbf{0.230}\,$\pm$\,0.008 & 51m 00s \\
\midrule
\multicolumn{8}{l}{\cellcolor{blue!4} \slr} \\
OATS                     & 20.13 & 13.96 & 0.892\,$\pm$\,0.003 & 0.889\,$\pm$\,0.007 & 0.969\,$\pm$\,0.001 & 0.748\,$\pm$\,0.004 & 2h 13m \\
3BASiL                   & 20.15 & \underline{\textbf{13.72}} & 0.922\,$\pm$\,0.003 & 0.928\,$\pm$\,0.005 & 0.980\,$\pm$\,0.001 & 0.819\,$\pm$\,0.003 & 3h 03m \\
\rowcolor{darkgreen!10}
\texttt{CACTUS-slr}        & \underline{\textbf{20.06}} & 13.73 & \underline{\textbf{0.937}}\,$\pm$\,0.003 & \underline{\textbf{0.945}}\,$\pm$\,0.004 & \underline{\textbf{0.987}}\,$\pm$\,0.001 & \underline{\textbf{0.859}}\,$\pm$\,0.004 & 3h 50m \\
\midrule[\heavyrulewidth]
\multicolumn{8}{c}{\cellcolor{blue!8}\textbf{50\% compression}} \\
\midrule
\multicolumn{8}{l}{\cellcolor{blue!4}\emph{Pure pruning}} \\
SparseGPT                & 24.43 & 18.04 & 0.656\,$\pm$\,0.004 & 0.660\,$\pm$\,0.008 & 0.874\,$\pm$\,0.002 & 0.551\,$\pm$\,0.005 & 24m 07s \\
ALPS                     & 22.60 & 16.40 & 0.709\,$\pm$\,0.004 & 0.710\,$\pm$\,0.008 & 0.895\,$\pm$\,0.002 & 0.585\,$\pm$\,0.005 & 43m 24s \\
LLM-Surgeon(5sh)         & 23.50 & 19.80 & 0.680\,$\pm$\,0.005 & 0.670\,$\pm$\,0.015 & 0.864\,$\pm$\,0.006 & 0.561\,$\pm$\,0.007 & 20h 13m$^{\dagger}$ \\
\rowcolor{darkgreen!10}
\texttt{CACTUS-prune}      & \textbf{21.85} & \textbf{16.27} & \textbf{0.762}\,$\pm$\,0.004 & \textbf{0.755}\,$\pm$\,0.010 & \textbf{0.927}\,$\pm$\,0.002 & \textbf{0.697}\,$\pm$\,0.005 & 42m 18s \\
\midrule
\multicolumn{8}{l}{\cellcolor{blue!4}\emph{Pure low-rank}} \\
SVD-LLM(W)               & 176.58 & 213.73 & 0.113\,$\pm$\,0.002 & 0.118\,$\pm$\,0.005 & 0.193\,$\pm$\,0.004 & 0.052\,$\pm$\,0.003 & 53m 15s \\
IPSVD                    & \textbf{94.19} & \textbf{167.49} & \textbf{0.146}\,$\pm$\,0.003 & \textbf{0.151}\,$\pm$\,0.007 & \textbf{0.354}\,$\pm$\,0.008 & \textbf{0.143}\,$\pm$\,0.004 & 53m 27s \\
\midrule
\multicolumn{8}{l}{\cellcolor{blue!4} \slr} \\
OATS                     & 23.97 & 16.73 & 0.652\,$\pm$\,0.003 & 0.653\,$\pm$\,0.009 & 0.804\,$\pm$\,0.002 & 0.508\,$\pm$\,0.004 & 2h 05m \\
3BASiL                   & 21.63 & \underline{\textbf{15.34}} & 0.743\,$\pm$\,0.004 & 0.747\,$\pm$\,0.008 & 0.919\,$\pm$\,0.001 & 0.695\,$\pm$\,0.004 & 2h 37m \\
\rowcolor{darkgreen!10}
\texttt{CACTUS-slr}        & \underline{\textbf{21.38}} & 15.59 & \underline{\textbf{0.782}}\,$\pm$\,0.004 & \underline{\textbf{0.790}}\,$\pm$\,0.009 & \underline{\textbf{0.937}}\,$\pm$\,0.001 & \underline{\textbf{0.698}}\,$\pm$\,0.005 & 2h 54m \\
\midrule[\heavyrulewidth]
\multicolumn{8}{c}{\cellcolor{blue!8}\textbf{70\% compression}} \\
\midrule
\multicolumn{8}{l}{\cellcolor{blue!4}\emph{Pure pruning}} \\
SparseGPT                & 47.88 & 48.97 & 0.303\,$\pm$\,0.003 & 0.298\,$\pm$\,0.011 & 0.424\,$\pm$\,0.009 & 0.128\,$\pm$\,0.004 & 24m 01s \\
ALPS                     & 34.95 & 32.32 & 0.406\,$\pm$\,0.004 & 0.397\,$\pm$\,0.011 & 0.534\,$\pm$\,0.008 & 0.070\,$\pm$\,0.003 & 44m 15s \\
LLM-Surgeon(5sh)         & 49.54 & 62.56 & 0.325\,$\pm$\,0.004 & 0.323\,$\pm$\,0.012 & 0.524\,$\pm$\,0.025 & 0.225\,$\pm$\,0.008 & 22h 01m$^{\dagger}$ \\
\rowcolor{darkgreen!10}
\texttt{CACTUS-prune}      & \textbf{30.73} & \textbf{29.19} & \textbf{0.475}\,$\pm$\,0.004 & \underline{\textbf{0.468}}\,$\pm$\,0.014 & \textbf{0.666}\,$\pm$\,0.015 & \underline{\textbf{0.303}}\,$\pm$\,0.008 & 47m 59s \\
\midrule
\multicolumn{8}{l}{\cellcolor{blue!4}\emph{Pure low-rank}} \\
SVD-LLM(W)               & \textbf{316.85} & 521.24 & 0.101\,$\pm$\,0.002 & 0.112\,$\pm$\,0.005 & 0.161\,$\pm$\,0.004 & \textbf{0.062}\,$\pm$\,0.003 & 52m 20s \\
IPSVD                    & 328.27 & \textbf{489.76} & \textbf{0.147}\,$\pm$\,0.002 & \textbf{0.154}\,$\pm$\,0.006 & \textbf{0.183}\,$\pm$\,0.005 & 0.061\,$\pm$\,0.003 & 49m 25s \\
\midrule
\multicolumn{8}{l}{\cellcolor{blue!4} \slr} \\
OATS                     & 89.01 & 68.15 & 0.064\,$\pm$\,0.002 & 0.059\,$\pm$\,0.005 & 0.328\,$\pm$\,0.009 & 0.113\,$\pm$\,0.005 & 2h 07m \\
3BASiL                   & 34.32 & 31.28 & 0.408\,$\pm$\,0.004 & 0.400\,$\pm$\,0.012 & 0.550\,$\pm$\,0.010 & 0.186\,$\pm$\,0.005 & 2h 11m \\
\rowcolor{darkgreen!10}
\texttt{CACTUS-slr}        & \underline{\textbf{30.21}} & \underline{\textbf{27.83}} & \underline{\textbf{0.477}}\,$\pm$\,0.004 & \underline{\textbf{0.468}}\,$\pm$\,0.013 & \underline{\textbf{0.674}}\,$\pm$\,0.012 & \textbf{0.275}\,$\pm$\,0.008 & 2h 14m \\
\bottomrule
\end{tabular}%
}
\vspace{3pt}
\caption{\textbf{Qwen3-4B} influence-preservation results in the \textbf{TracIn} representation.}
\label{tab:main-sweep-appendix-tracin-q4b}
\end{table}

\clearpage

\begin{table}[H]
\centering
\scriptsize
\setlength{\tabcolsep}{1.5pt}

\newcolumntype{C}[1]{>{\centering\arraybackslash}p{#1}}
\resizebox{0.90\textwidth}{!}{%
\begin{tabular}{l|C{3.0em}C{3.0em}|C{7.1em}C{7.1em}C{7.1em}C{7.1em}|C{4.2em}}
\toprule
\textbf{Method}
& \cellcolor{gray!10}\textbf{C4 PPL} ($\downarrow$)
& \cellcolor{gray!10}\textbf{WT-2 PPL} ($\downarrow$)
& \multicolumn{4}{c|}{\textbf{KFAC} ($\uparrow$)}
& \multirow{2}{*}{\shortstack{\textbf{Compress}\\\textbf{Runtime}}} \\
\cmidrule(lr){4-7}
& \cellcolor{gray!10}
& \cellcolor{gray!10}
& $\cos_{\text{tr}}$ & $\cos_{\text{te}}$ & $\rho_{\text{qry}}$ & $R_{\text{qry@5\%}}$ & \\
\midrule
\rowcolor{gray!10} \textbf{Dense} & 19.85 & 13.66 & 1.000 & 1.000 & 1.000 & 1.000 & --- \\
\midrule[\heavyrulewidth]
\multicolumn{8}{c}{\cellcolor{blue!8}\textbf{30\% compression}} \\
\midrule
\multicolumn{8}{l}{\cellcolor{blue!4}\emph{Pure pruning}} \\
SparseGPT                & 21.05 & 14.53 & 0.739\,$\pm$\,0.004 & 0.750\,$\pm$\,0.010 & 0.498\,$\pm$\,0.008 & 0.292\,$\pm$\,0.008 & 24m 53s \\
ALPS                     & 20.44 & \textbf{13.84} & 0.794\,$\pm$\,0.004 & 0.805\,$\pm$\,0.009 & 0.595\,$\pm$\,0.007 & 0.375\,$\pm$\,0.006 & 43m 07s \\
LLM-Surgeon(5sh)         & 20.44 & 14.33 & \textbf{0.825}\,$\pm$\,0.004 & \textbf{0.836}\,$\pm$\,0.009 & \textbf{0.630}\,$\pm$\,0.010 & 0.373\,$\pm$\,0.012 & 18h 41m$^{\dagger}$ \\
\rowcolor{darkgreen!10}
\texttt{CACTUS-prune}      & \textbf{20.21} & 13.87 & 0.816\,$\pm$\,0.004 & 0.816\,$\pm$\,0.009 & 0.624\,$\pm$\,0.008 & \textbf{0.388}\,$\pm$\,0.007 & 45m 59s \\
\midrule
\multicolumn{8}{l}{\cellcolor{blue!4}\emph{Pure low-rank}} \\
SVD-LLM(W)               & 47.31 & 70.50 & 0.075\,$\pm$\,0.001 & 0.079\,$\pm$\,0.004 & 0.018\,$\pm$\,0.004 & 0.075\,$\pm$\,0.004 & 54m 43s \\
IPSVD                    & \textbf{37.79} & \textbf{60.98} & \textbf{0.141}\,$\pm$\,0.002 & \textbf{0.140}\,$\pm$\,0.005 & \textbf{0.036}\,$\pm$\,0.004 & \textbf{0.081}\,$\pm$\,0.004 & 51m 00s \\
\midrule
\multicolumn{8}{l}{\cellcolor{blue!4} \slr} \\
OATS                     & 20.13 & 13.96 & 0.790\,$\pm$\,0.003 & 0.788\,$\pm$\,0.009 & 0.588\,$\pm$\,0.007 & 0.360\,$\pm$\,0.007 & 2h 13m \\
3BASiL                   & 20.15 & \underline{\textbf{13.72}} & 0.857\,$\pm$\,0.004 & 0.858\,$\pm$\,0.010 & 0.698\,$\pm$\,0.008 & 0.441\,$\pm$\,0.008 & 3h 03m \\
\rowcolor{darkgreen!10}
\texttt{CACTUS-slr}        & \underline{\textbf{20.06}} & 13.73 & \underline{\textbf{0.866}}\,$\pm$\,0.004 & \underline{\textbf{0.871}}\,$\pm$\,0.008 & \underline{\textbf{0.715}}\,$\pm$\,0.007 & \underline{\textbf{0.460}}\,$\pm$\,0.008 & 3h 50m \\
\midrule[\heavyrulewidth]
\multicolumn{8}{c}{\cellcolor{blue!8}\textbf{50\% compression}} \\
\midrule
\multicolumn{8}{l}{\cellcolor{blue!4}\emph{Pure pruning}} \\
SparseGPT                & 24.43 & 18.04 & 0.387\,$\pm$\,0.003 & 0.388\,$\pm$\,0.008 & 0.138\,$\pm$\,0.004 & 0.128\,$\pm$\,0.005 & 24m 07s \\
ALPS                     & 22.60 & 16.40 & 0.458\,$\pm$\,0.004 & 0.455\,$\pm$\,0.010 & 0.190\,$\pm$\,0.005 & \textbf{0.150}\,$\pm$\,0.006 & 43m 24s \\
LLM-Surgeon(5sh)         & 23.50 & 19.80 & 0.397\,$\pm$\,0.004 & 0.383\,$\pm$\,0.013 & 0.140\,$\pm$\,0.005 & 0.125\,$\pm$\,0.005 & 20h 13m$^{\dagger}$ \\
\rowcolor{darkgreen!10}
\texttt{CACTUS-prune}      & \textbf{21.85} & \textbf{16.27} & \textbf{0.493}\,$\pm$\,0.005 & \textbf{0.474}\,$\pm$\,0.013 & \textbf{0.211}\,$\pm$\,0.006 & 0.149\,$\pm$\,0.006 & 42m 18s \\
\midrule
\multicolumn{8}{l}{\cellcolor{blue!4}\emph{Pure low-rank}} \\
SVD-LLM(W)               & 176.58 & 213.73 & 0.007\,$\pm$\,0.001 & 0.009\,$\pm$\,0.001 & 0.013\,$\pm$\,0.004 & 0.047\,$\pm$\,0.004 & 53m 15s \\
IPSVD                    & \textbf{94.19} & \textbf{167.49} & \textbf{0.030}\,$\pm$\,0.001 & \textbf{0.032}\,$\pm$\,0.002 & \textbf{0.021}\,$\pm$\,0.004 & \textbf{0.062}\,$\pm$\,0.004 & 53m 27s \\
\midrule
\multicolumn{8}{l}{\cellcolor{blue!4} \slr} \\
OATS                     & 23.97 & 16.73 & 0.380\,$\pm$\,0.003 & 0.383\,$\pm$\,0.009 & 0.140\,$\pm$\,0.004 & 0.130\,$\pm$\,0.005 & 2h 05m \\
3BASiL                   & 21.63 & \underline{\textbf{15.34}} & 0.521\,$\pm$\,0.004 & 0.521\,$\pm$\,0.011 & 0.252\,$\pm$\,0.005 & \underline{\textbf{0.187}}\,$\pm$\,0.006 & 2h 37m \\
\rowcolor{darkgreen!10}
\texttt{CACTUS-slr}        & \underline{\textbf{21.38}} & 15.59 & \underline{\textbf{0.539}}\,$\pm$\,0.005 & \underline{\textbf{0.538}}\,$\pm$\,0.013 & \underline{\textbf{0.259}}\,$\pm$\,0.007 & 0.173\,$\pm$\,0.006 & 2h 54m \\
\midrule[\heavyrulewidth]
\multicolumn{8}{c}{\cellcolor{blue!8}\textbf{70\% compression}} \\
\midrule
\multicolumn{8}{l}{\cellcolor{blue!4}\emph{Pure pruning}} \\
SparseGPT                & 47.88 & 48.97 & 0.079\,$\pm$\,0.001 & 0.082\,$\pm$\,0.004 & 0.024\,$\pm$\,0.004 & 0.063\,$\pm$\,0.003 & 24m 01s \\
ALPS                     & 34.95 & 32.32 & 0.147\,$\pm$\,0.002 & 0.145\,$\pm$\,0.005 & 0.042\,$\pm$\,0.004 & 0.074\,$\pm$\,0.004 & 44m 15s \\
LLM-Surgeon(5sh)         & 49.54 & 62.56 & 0.085\,$\pm$\,0.001 & 0.085\,$\pm$\,0.003 & 0.025\,$\pm$\,0.004 & \underline{\textbf{0.094}}\,$\pm$\,0.005 & 22h 01m$^{\dagger}$ \\
\rowcolor{darkgreen!10}
\texttt{CACTUS-prune}      & \textbf{30.73} & \textbf{29.19} & \underline{\textbf{0.174}}\,$\pm$\,0.003 & \underline{\textbf{0.172}}\,$\pm$\,0.007 & \underline{\textbf{0.044}}\,$\pm$\,0.004 & 0.087\,$\pm$\,0.004 & 47m 59s \\
\midrule
\multicolumn{8}{l}{\cellcolor{blue!4}\emph{Pure low-rank}} \\
SVD-LLM(W)               & \textbf{316.85} & 521.24 & 0.002\,$\pm$\,0.001 & 0.003\,$\pm$\,0.001 & \textbf{0.008}\,$\pm$\,0.004 & \textbf{0.050}\,$\pm$\,0.003 & 52m 20s \\
IPSVD                    & 328.27 & \textbf{489.76} & \textbf{0.003}\,$\pm$\,0.001 & \textbf{0.006}\,$\pm$\,0.001 & 0.000\,$\pm$\,0.004 & \textbf{0.050}\,$\pm$\,0.003 & 49m 25s \\
\midrule
\multicolumn{8}{l}{\cellcolor{blue!4} \slr} \\
OATS                     & 89.01 & 68.15 & 0.024\,$\pm$\,0.001 & 0.026\,$\pm$\,0.002 & 0.026\,$\pm$\,0.004 & 0.060\,$\pm$\,0.004 & 2h 07m \\
3BASiL                   & 34.32 & 31.28 & 0.150\,$\pm$\,0.002 & 0.146\,$\pm$\,0.006 & 0.032\,$\pm$\,0.004 & 0.071\,$\pm$\,0.004 & 2h 11m \\
\rowcolor{darkgreen!10}
\texttt{CACTUS-slr}        & \underline{\textbf{30.21}} & \underline{\textbf{27.83}} & \textbf{0.173}\,$\pm$\,0.002 & \textbf{0.167}\,$\pm$\,0.006 & \textbf{0.038}\,$\pm$\,0.004 & \textbf{0.082}\,$\pm$\,0.004 & 2h 14m \\
\bottomrule
\end{tabular}%
}
\vspace{3pt}
\caption{\textbf{Qwen3-4B.} Influence-preservation results in the \textbf{KFAC} representation.}
\label{tab:main-sweep-appendix-kfac-q4b}
\end{table}

\clearpage

\paragraph{Llama-3-8B at $(2{:}4{+}r256)$.} \Cref{tab:main-sweep-appendix-l8b-24-r256} reports the four influence-preservation metrics for the Llama-3-8B \slr configuration ($2{:}4{+}r256$) alongside C4 and WikiText-2 perplexities. Evaluation protocol matches \cref{sect:exp-setup}: $n_{\text{tr}}=1000$ C4 training samples and $n_{\text{te}}=100$ WikiText-2 test samples at sequence length $512$, projection dim $d=8192$. \textbf{Bold}: best across the three reported \slr methods.

\begin{table}[H]
\centering
\scriptsize
\setlength{\tabcolsep}{1.5pt}

\newcolumntype{C}[1]{>{\centering\arraybackslash}p{#1}}
\resizebox{0.75\textwidth}{!}{%
\begin{tabular}{l|C{3.0em}C{3.0em}|C{7.1em}C{7.1em}C{7.1em}C{7.1em}|C{4.2em}}
\toprule
\textbf{Method}
& \cellcolor{gray!10}\textbf{C4 PPL} ($\downarrow$)
& \cellcolor{gray!10}\textbf{WT-2 PPL} ($\downarrow$)
& \multicolumn{4}{c|}{\textbf{TracIn} ($\uparrow$)}
& \multirow{2}{*}{\shortstack{\textbf{Compress}\\\textbf{Runtime}}} \\
\cmidrule(lr){4-7}
& \cellcolor{gray!10}
& \cellcolor{gray!10}
& $\cos_{\text{tr}}$ & $\cos_{\text{te}}$ & $\rho_{\text{qry}}$ & $R_{\text{qry@5\%}}$ & \\
\midrule
\rowcolor{gray!10} \textbf{Dense} & 9.45 & 6.14 & 1.000 & 1.000 & 1.000 & 1.000 & --- \\
\midrule[\heavyrulewidth]
\multicolumn{8}{c}{\cellcolor{blue!8}\textbf{$2{:}4 + r256$ compression}} \\
\midrule
\multicolumn{8}{l}{\cellcolor{blue!4} \slr} \\
OATS               & 13.85 & 9.02 & 0.391\,$\pm$\,0.004 & 0.393\,$\pm$\,0.013 & 0.093\,$\pm$\,0.009 & 0.093\,$\pm$\,0.005 & 5h 28m \\
3BASiL             & 12.27 & 8.06 & 0.518\,$\pm$\,0.005 & 0.523\,$\pm$\,0.016 & 0.170\,$\pm$\,0.016 & 0.148\,$\pm$\,0.011 & \textbf{2h 42m} \\
\rowcolor{darkgreen!10}
\texttt{\ourmethod-SLR}
& \textbf{12.21} & \textbf{8.04} & \textbf{0.573}\,$\pm$\,0.006 & \textbf{0.569}\,$\pm$\,0.019 & \textbf{0.236}\,$\pm$\,0.026 & \textbf{0.224}\,$\pm$\,0.018 & 4h 07m \\
\bottomrule
\end{tabular}%
}
\vspace{3pt}
\caption{\textbf{Llama-3-8B.} Influence-preservation results at the $(2{:}4{+}r256)$ \slr configuration in the \textbf{TracIn} (unweighted) representation. Each cell is annotated with $\pm\,\widehat{\mathrm{SEM}}$ following the same convention as Tables~\ref{tab:main-sweep-appendix-tracin}--\ref{tab:main-sweep-appendix-kfac}. Compression runtime is measured on a single A100 GPU.}
\label{tab:main-sweep-appendix-l8b-24-r256}
\end{table}

\paragraph{Llama-3-8B perplexity results across $2{:}4$ and $2{:}4{+}r$ \slr configurations.} \Cref{tab:main-sweep-appendix-l8b-ppl} reports C4, WikiText-2, and PTB~\citep{Marcus1994} perplexities ($\downarrow$) across three compression budgets: $2{:}4$ pure pruning, and the \slr settings $(2{:}4{+}r64)$ and $(2{:}4{+}r256)$. For each budget, three methods are reported in order: activation-only: SparseGPT and ALPS for pure pruning, followed by OATS and 3BASiL for \slr, and curvature-aware: \ourmethod.

\begin{table}[H]
\centering
\scriptsize
\setlength{\tabcolsep}{2.2pt}
\newcolumntype{C}[1]{>{\centering\arraybackslash}p{#1}}
\resizebox{0.7\textwidth}{!}{%
\begin{tabular}{l|C{2.2em}C{2.2em}C{2.2em}|l|C{2.2em}C{2.2em}C{2.2em}|C{2.2em}C{2.2em}C{2.2em}}
\toprule
\rowcolor{gray!10} \textbf{Dense} & \multicolumn{10}{c}{C4 $=$ 9.45 \quad WT-2 $=$ 6.14 \quad PTB $=$ 11.18} \\
\midrule
\textbf{Method}
& \multicolumn{3}{c|}{\cellcolor{gray!10}\textbf{$2{:}4$} pruning}
& \textbf{Method}
& \multicolumn{3}{c|}{\cellcolor{gray!10}\textbf{$(2{:}4{+}r64)$}}
& \multicolumn{3}{c}{\cellcolor{gray!10}\textbf{$(2{:}4{+}r256)$}} \\
\cmidrule(lr){2-4}\cmidrule(lr){6-8}\cmidrule(lr){9-11}
& \tiny{C4} & \tiny{WT-2} & \tiny{PTB} & & \tiny{C4} & \tiny{WT-2} & \tiny{PTB} & \tiny{C4} & \tiny{WT-2} & \tiny{PTB} \\
\midrule
\texttt{SparseGPT}      & 22.27 & 15.98 & 24.30 & \texttt{OATS}        & 22.28 & 15.36 & 23.33 & 13.85 & 9.02 & 14.47 \\
\texttt{ALPS}           & 20.08 & 14.44 & 20.62 & \texttt{3BASiL}      & 15.84 & 11.20 & 16.10 & 12.27 & 8.06 & \textbf{12.73} \\
\rowcolor{darkgreen!10}
\texttt{\ourmethod-prune}
& \textbf{19.74} & \textbf{14.37} & \textbf{20.12}
& \texttt{\ourmethod-SLR}
& \textbf{15.46} & \textbf{10.75} & \textbf{15.73}
& \textbf{12.21} & \textbf{8.04} & 12.80 \\
\bottomrule
\end{tabular}%
}
\vspace{3pt}
\caption{Llama-3-8B perplexity ($\downarrow$) for one-shot $2{:}4$ pruning, and $(2{:}4{+}r)$ \slr compression. \textbf{Bold}: best across the three reported methods within each compression budget.}
\label{tab:main-sweep-appendix-l8b-ppl}
\end{table}

\paragraph{Llama-3.2-3B unstructured pruning.} \Cref{tab:main-sweep-appendix-l3b-ppl} reports C4 and WikiText-2 perplexities ($\downarrow$) for unstructured one-shot pruning on Llama-3.2-3B across nine sparsity levels.

\begin{table}[H]
\centering
\scriptsize
\setlength{\tabcolsep}{2.2pt}

\newcolumntype{C}[1]{>{\centering\arraybackslash}p{#1}}
\resizebox{\textwidth}{!}{%
\begin{tabular}{l|C{2.3em}C{2.3em}|C{2.3em}C{2.3em}|C{2.3em}C{2.3em}|C{2.3em}C{2.3em}|C{2.3em}C{2.3em}|C{2.3em}C{2.3em}|C{2.3em}C{2.3em}|C{2.4em}C{2.6em}|C{2.6em}C{3.0em}}
\toprule
\rowcolor{gray!10} \textbf{Dense} & \multicolumn{18}{c}{C4 $=$ 11.33 \quad WT-2 $=$ 7.81} \\
\midrule
\textbf{Method}
& \multicolumn{2}{c|}{\cellcolor{gray!10}\textbf{$10\%$}}
& \multicolumn{2}{c|}{\cellcolor{gray!10}\textbf{$20\%$}}
& \multicolumn{2}{c|}{\cellcolor{gray!10}\textbf{$30\%$}}
& \multicolumn{2}{c|}{\cellcolor{gray!10}\textbf{$40\%$}}
& \multicolumn{2}{c|}{\cellcolor{gray!10}\textbf{$50\%$}}
& \multicolumn{2}{c|}{\cellcolor{gray!10}\textbf{$60\%$}}
& \multicolumn{2}{c|}{\cellcolor{gray!10}\textbf{$70\%$}}
& \multicolumn{2}{c|}{\cellcolor{gray!10}\textbf{$80\%$}}
& \multicolumn{2}{c}{\cellcolor{gray!10}\textbf{$90\%$}} \\
\cmidrule(lr){2-3}\cmidrule(lr){4-5}\cmidrule(lr){6-7}\cmidrule(lr){8-9}\cmidrule(lr){10-11}\cmidrule(lr){12-13}\cmidrule(lr){14-15}\cmidrule(lr){16-17}\cmidrule(lr){18-19}
& {\tiny C4} & {\tiny WT-2} & {\tiny C4} & {\tiny WT-2} & {\tiny C4} & {\tiny WT-2} & {\tiny C4} & {\tiny WT-2} & {\tiny C4} & {\tiny WT-2} & {\tiny C4} & {\tiny WT-2} & {\tiny C4} & {\tiny WT-2} & {\tiny C4} & {\tiny WT-2} & {\tiny C4} & {\tiny WT-2} \\
\midrule
\texttt{SparseGPT}      & 11.42 & 7.86 & 11.68 & 8.01 & 12.34 & 8.42 & 13.80 & 9.37 & 17.31 & 12.24 & 29.38 & 21.73 & 77.82 & 69.13 & 212.5 & 261.5 & 999.1 & 1510.6 \\
\texttt{ALPS}           & 11.39 & 7.85 & 11.61 & 7.98 & 12.18 & 8.33 & 13.40 & 9.14 & 15.85 & 11.30 & 22.92 & 18.19 & 45.57 & 42.39 & 106.5 & 132.9 & 365.7 & 563.6 \\
\rowcolor{darkgreen!10}
\texttt{\ourmethod-prune}
& \textbf{11.38} & \textbf{7.84} & \textbf{11.56} & \textbf{7.94} & \textbf{12.07} & \textbf{8.27} & \textbf{13.17} & \textbf{9.01} & \textbf{15.36} & \textbf{11.12} & \textbf{21.49} & \textbf{17.37} & \textbf{41.92} & \textbf{35.58} & \textbf{101.2} & \textbf{125.2} & \textbf{354.7} & \textbf{482.7} \\
\bottomrule
\end{tabular}%
}
\vspace{3pt}
\caption{Llama-3.2-3B perplexity ($\downarrow$) for one-shot unstructured pruning across nine sparsity levels. \textbf{Bold}: best across the three reported methods within each sparsity level.}
\label{tab:main-sweep-appendix-l3b-ppl}
\end{table}

\begin{table}[H]
\centering
\scriptsize
\setlength{\tabcolsep}{3pt}
\resizebox{\textwidth}{!}{%
\begin{tabular}{l|rr|rrrrrrr|r|r}
\toprule
Method & C4 & WT-2 & PIQA & ARC-E & ARC-C & HellaSwag & WinoGrande & OBQA & RTE & Avg. & Time \\
\midrule
\rowcolor{gray!10}
Dense & 11.27 & 7.05 & 81.72 & 73.06 & 58.96 & 82.91 & 75.85 & 46.40 & 82.31 & 71.60 & --- \\
SparseGPT & 14.17 & 9.47 & 78.62 & 80.13 & 56.91 & 74.18 & \textbf{72.85} & 40.40 & 75.45 & 68.36 & \textbf{58m 15s} \\
\rowcolor{darkgreen!10}
\ourmethod-prune & \textbf{13.42} & \textbf{8.88} & \textbf{79.27} & \textbf{82.07} & \textbf{58.11} & \textbf{75.51} & 72.77 & \textbf{41.80} & \textbf{78.34} & \textbf{69.70} & 9h 28m \\
\bottomrule
\end{tabular}%
}
\caption{\textbf{Qwen3.8-27B at $2{:}4$ sparsity} across all $496$ decoder linear matrices: C4/WT-2 perplexity ($\downarrow$), seven zero-shot accuracies ($\uparrow$). \textbf{Bold}: better compressed method. Total time ($\downarrow$) uses using $2$ NVIDIA L40S GPUs for Qwen3.8-27B model.}
\label{tab:qwen38-27b-2to4}
\end{table}

\clearpage
\subsection{LESS selection scores distribution}
\begin{figure}[H]
\centering
\includegraphics[width=0.90\textwidth]{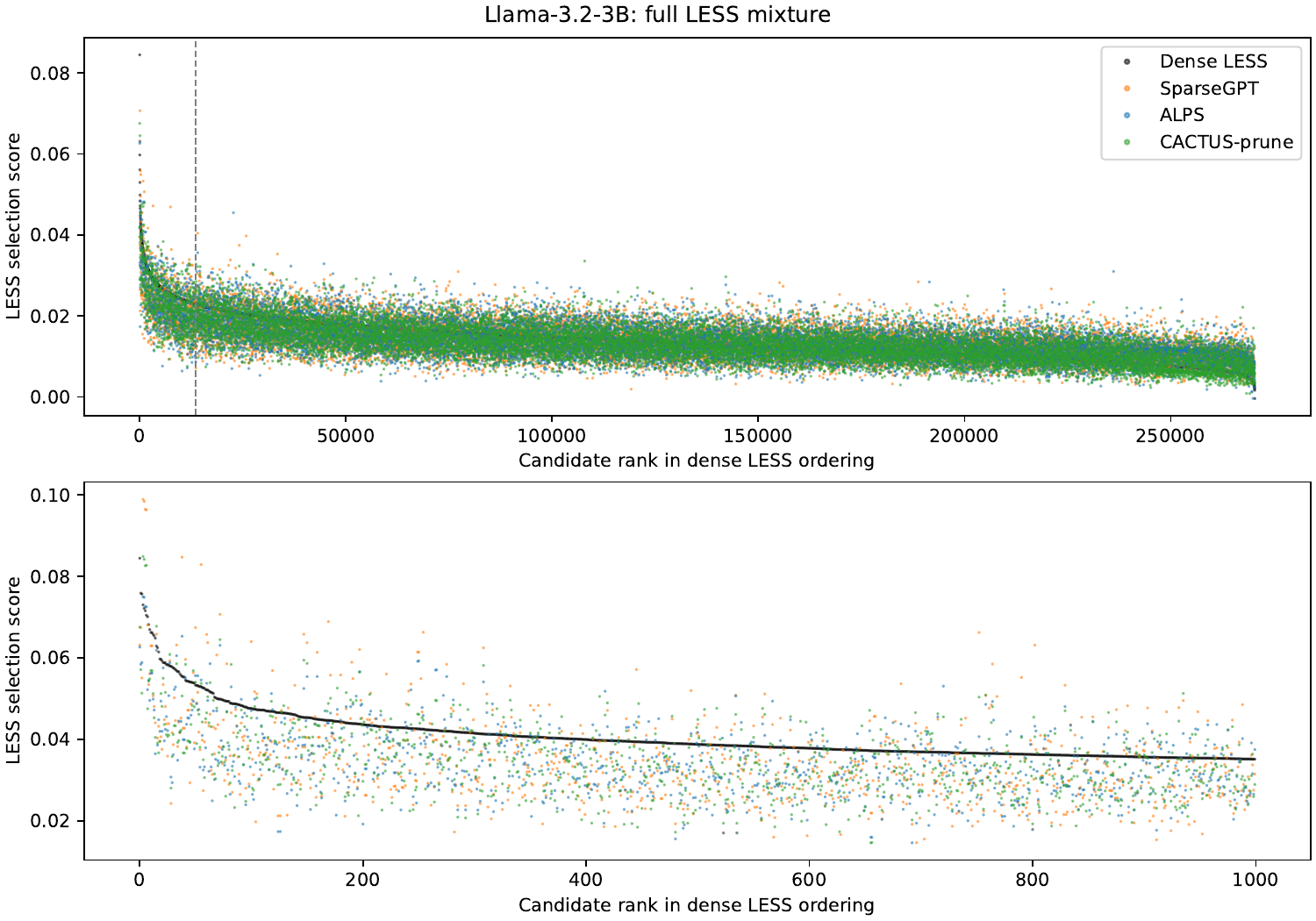}
\caption{LESS selection scores for dense Llama-3.2-3B and its SparseGPT, ALPS, and \ourmethod-prune sparse proxies, with candidates ordered by the dense model's aggregate score. The top panel spans the full instruction mixture; the bottom panel shows the first $1000$ candidates. The dashed line marks the top-$5\%$ cutoff.}
\label{fig:less-bbh-score-curve}
\end{figure}

\begin{figure}[H]
\centering
\includegraphics[width=0.90\textwidth]{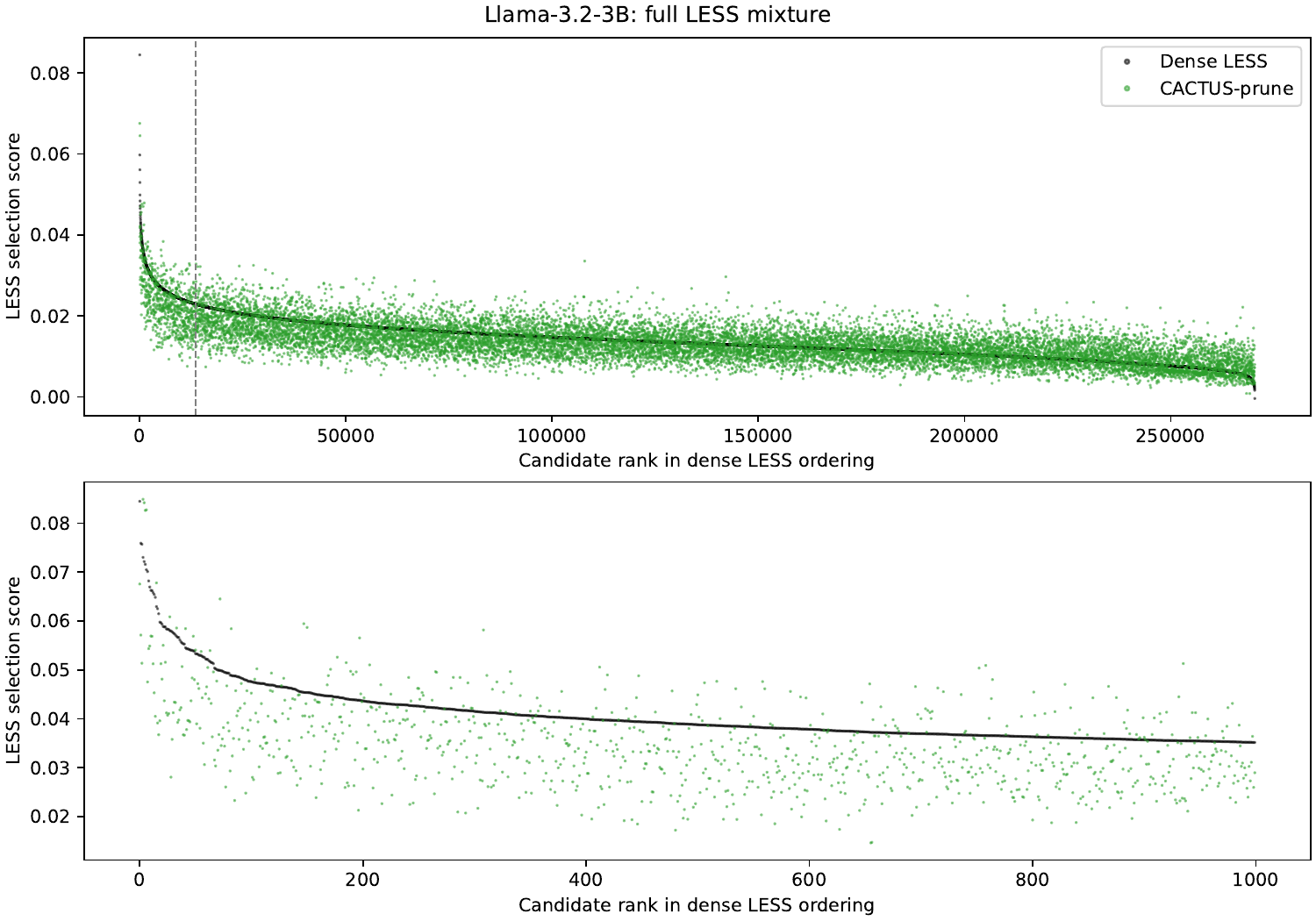}
\caption{Dense LESS and \ourmethod-prune scores from \cref{fig:less-bbh-score-curve}, shown separately.}
\label{fig:less-bbh-score-curve-dense-cactus}
\end{figure}

\subsection{Per-layer Loss Decrease}
We examine how the layerwise KFAC objective evolves over \ourmethod's ADMM iterations. \Cref{fig:kfac-loss-curve-layer0} compares representative loss trajectories for pure pruning and \slr compression, illustrating their initial non-monotonic behavior and subsequent stabilization as the penalty increases.

\begin{figure}[H]
\centering
\includegraphics[width=0.9\textwidth]{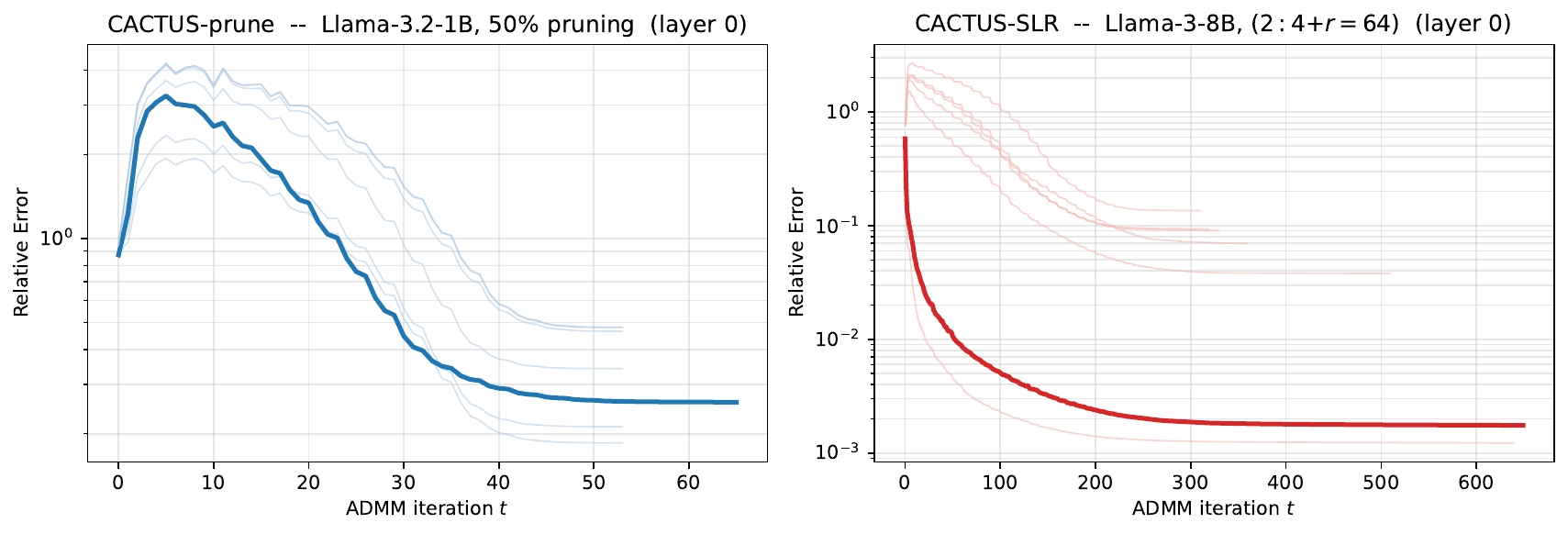}
\caption{\ourmethod relative-error trajectory for a representative layer. The KFAC objective can increase during the first ADMM iterations because the method is exploring feasible sparse candidates while the penalty schedule is still loose. As the penalty grows, primal feasibility tightens and the iterates settle, yielding the expected decrease and convergence of the objective.}
\label{fig:kfac-loss-curve-layer0}
\end{figure}

\subsection{Effect of Calibration-Set Size}\label{app:calibration-size}
We vary the number of C4 calibration sequences from $32$ to $256$ for \ourmethod-prune at $50\%$ unstructured sparsity on Llama-3.2-1B and Qwen3-4B (\cref{tab:calibration-size}). All runs use sequence length $2048$ and same calibration data and ridge coefficients $\lambda_H=0.01$ and $\lambda_G=0.1$. Perplexity improves with more calibration data on all three datasets. Doubling the sample count from $128$ to $256$ yields modest improvements, supporting $128$ sequences as a practical calibration budget. 

\begin{table}[H]
\centering
\small
\begin{tabular}{r rrr rrr}
\toprule
& \multicolumn{3}{c}{Llama-3.2-1B} & \multicolumn{3}{c}{Qwen3-4B} \\
\cmidrule(lr){2-4}\cmidrule(lr){5-7}
Sequences & C4 & WT-2 & PTB & C4 & WT-2 & PTB \\
\midrule
32  & 21.113 & 15.455 & 24.121 & 22.429 & 17.235 & 28.703 \\
64  & 20.105 & 14.566 & 23.181 & 21.952 & 16.811 & 28.576 \\
128 & 19.465 & 14.211 & 22.613 & 21.743 & 16.102 & 27.924 \\
256 & 19.132 & 13.842 & 22.399 & 21.663 & 16.057 & 27.660 \\
\bottomrule
\end{tabular}
\caption{Effect of calibration-set size on perplexity ($\downarrow$) for \ourmethod-prune at $50\%$ unstructured sparsity. All other settings are held fixed for each model.}
\label{tab:calibration-size}
\end{table}

\end{document}